\documentclass{article}

\usepackage{evotreenad_preprint,times}

\usepackage[utf8]{inputenc} 
\usepackage[T1]{fontenc}    
\usepackage{graphicx}       
\usepackage[hidelinks]{hyperref}       
\usepackage{url}            
\usepackage{booktabs}       
\usepackage{amsmath}        
\usepackage{amsfonts}       
\usepackage{amssymb}        
\usepackage{amsthm}         
\usepackage{algorithm}      
\usepackage{algorithmic}    
\usepackage{fancyvrb}       
\usepackage{fvextra}        
\usepackage{multirow}       
\usepackage{subcaption}     
\usepackage{placeins}       
\usepackage{nicefrac}       
\usepackage{microtype}      
\usepackage[table]{xcolor}  

\makeatletter
\renewcommand{\paragraph}{\@startsection{paragraph}{4}{\z@}
  {0.4ex plus 0.1ex minus 0.1ex}{-1em}{\normalsize\bfseries}}
\makeatother

\definecolor{OursBlue}{HTML}{E7F1FB}
\definecolor{StdBlueGray}{HTML}{5F6F86}
\newcommand{\std}[1]{{\color{StdBlueGray}\tiny $\pm$#1}}

\DefineVerbatimEnvironment{CodeBlock}{Verbatim}{fontsize=\scriptsize,breaklines=true,breakanywhere=true}

\newtheorem{theorem}{Theorem}
\newtheorem{lemma}{Lemma}
\newtheorem{proposition}{Proposition}

\title{{\fontsize{16}{19}\selectfont EvoTreeNAD: Genealogy-Guided Evolution for LLM-Driven Neural Architecture Discovery}}

\author{Lishan Yu, Derek Jiu, Qizhen Lan, Xiaoqian Jiang\\
Department of Health Data Science and AI\\
McWilliams School of Biomedical Informatics\\
University of Texas Health Science Center at Houston}

\hypersetup{pdfauthor={Lishan Yu, Derek Jiu, Qizhen Lan, Xiaoqian Jiang}}

\hypersetup{pdftitle={EvoTreeNAD: Genealogy-Guided Evolution for LLM-Driven Neural Architecture Discovery}}

\begin{document}
\raggedbottom

\maketitle
\fancyhead{}
\renewcommand{\headrulewidth}{0pt}

\begin{abstract}

AI-driven scientific discovery accelerates research by autonomously
developing solutions and designs. Large language model (LLM) agents support
this process through iterative generation and evaluation. Yet these iterations
alone do not ensure cumulative progress or establish which directions to
pursue next. Costly evaluation further constrains the scope of exploration.
Neural architecture discovery brings these challenges together, coupling
open-ended design with resource-intensive experimentation.
We introduce \textbf{EvoTreeNAD}, a genealogy-guided evolutionary algorithm
that constructs trainable architectures without a supplied seed or a
hand-specified search space. Starting from an empty root, it grows a persistent
genealogy in which each new node represents a complete architecture.
Top-percentile values computed from each node and its descendants guide lineage
selection. Using the selected design history, an Idea Agent proposes a variant
and a Code Agent implements it. Each evaluated variant becomes a child node,
expanding the genealogy while providing evidence for subsequent lineage selection.
Our theoretical analysis establishes the existence of stationary variation regimes as the genealogy grows.
Under specified variation assumptions, sustained top-percentile family values
quantify the probability of generating high-reward architectures in these regimes.
EvoTreeNAD discovers architectures that outperform the compared NAS and NAD
baselines, achieving CIFAR-10/100 test errors of \(2.05{\pm}0.06\%\) and
\(15.09{\pm}0.22\%\). On all six MedMNIST-v2 tasks, the discovered
architectures surpass the strongest listed baselines.
A controlled CIFAR-10 study further shows that EvoTreeNAD outperforms
direct generation, best-of-\(N\) greedy continuation,
and full-family-mean routing.

\end{abstract}


\section{Introduction}

Large language model (LLM) agents are advancing AI-driven scientific discovery
and automated research by generating, implementing, and evaluating candidate
solutions with limited human intervention
\citep{lu2024aiscientist,novikov2025alphaevolve}.
Yet the ability to repeat this cycle does not ensure that individual findings
accumulate into sustained progress or establish which research directions to
pursue next. Costly evaluation makes both challenges more consequential,
limiting both the breadth of exploration and the depth to which promising
directions can be developed \citep{jiang2025aide}.

Existing approaches organize iterative design through repeated refinement,
candidate selection, and evolutionary search over evaluated solutions
\citep{novikov2025alphaevolve,jiang2025aide}.
These strategies support continued improvement, but candidate quality alone
does not fully reveal the potential for further development.
Strong designs may yield weak variants, while initially less successful designs
may become competitive through further modification.
Sustained discovery therefore requires building on successful designs while
using accumulated experimental evidence to revise earlier development choices.

Neural architecture discovery (NAD) makes this challenge concrete.
Recent LLM-based methods move beyond predefined search spaces by generating
and iteratively refining executable architectures
\citep{chen2023evoprompting,yang2025nader}.
Yet empirical evaluation of newly generated architectures can require
substantial computation, making it critical to decide not only what to
generate next, but which directions remain worth pursuing.

We introduce \textbf{EvoTreeNAD}, a genealogy-guided evolutionary algorithm
that treats architecture discovery as the continued development and
reassessment of design lineages.
EvoTreeNAD preserves complete trainable architectures and their branching
design histories in a persistent genealogy.
Starting from an empty root, it constructs architectures without a supplied
seed architecture or hand-specified search space.
Top-percentile family values summarize the outcomes accumulated by each node
and its descendants and guide the selection of a root-to-node lineage for
further evolution.
Using the selected design history, an Idea Agent proposes the next
architectural design or modification, and a Code Agent implements it.
Each evaluated variant extends the genealogy, contributing both a design for
further evolution and new evidence for subsequent lineage selection.
Successful architectures can thus support continued evolution while later
evidence can redirect development toward retained alternatives.

Our theoretical analysis establishes the existence of stationary empirical
variation regimes as the genealogy grows.
Under specified assumptions on lineage-conditioned variation, sustained
top-percentile family values exactly characterize the probability of
generating high-reward architectures in these regimes.
We further quantify how weak outcomes affect full-family means and establish
a pathwise equivalence between maximum-family routing and best-of-\(N\)
greedy continuation, a strategy used in automated design
\citep{ma2024eureka}.
Empirically, EvoTreeNAD discovers architectures that outperform the compared
NAS and NAD baselines on CIFAR-10/100 under fixed full-fidelity training.
The discovered architectures also surpass the strongest listed baseline on
each of six MedMNIST-v2 tasks spanning diverse medical imaging modalities
across 2D image and 3D volume classification.
Controlled CIFAR-10 experiments further demonstrate its advantage over
direct generation, best-of-\(N\) greedy continuation, and full-family-mean routing.

Our contributions are summarized as follows.
\begingroup
\setlength{\leftmargini}{1.5em}
\setlength{\labelsep}{0.4em}
\setlength{\labelwidth}{1.1em}
\begin{itemize}
\setlength{\itemsep}{1.2pt}
\setlength{\parskip}{0pt}
\setlength{\parsep}{0pt}
\setlength{\topsep}{1pt}
\item \textbf{Genealogy-guided architecture evolution.}
We introduce an evolutionary algorithm that preserves complete architectures
and their branching histories, combining lineage-conditioned LLM generation with
cumulative descendant evidence and reassessment of earlier design choices.
\item \textbf{Theoretical characterization of genealogy-guided evolution.}
We establish stationary empirical variation regimes for adaptive genealogy
growth and, under specified lineage-conditioned variation assumptions,
derive an exact relation between sustained top-percentile family values and
the long-run probability of generating high-reward architectures in these
regimes.
\item \textbf{Architecture performance and controlled evaluation.}
We discover architectures that outperform the compared baselines across
CIFAR and six MedMNIST-v2 tasks and demonstrate the advantage of EvoTreeNAD
through controlled comparisons with direct generation, greedy continuation,
and full-family-mean routing.
\end{itemize}
\endgroup


\section{Related Work}

\paragraph{LLM agents for automated discovery.}
LLM agents extend research automation from proposing solutions to implementing
experiments, evaluating outcomes, and developing subsequent proposals.
OPRO uses evaluated solutions to inform new proposals, while the AI Scientist
automates idea generation, experimentation, and reporting
\citep{yang2024opro,lu2024aiscientist}.
Other systems explicitly organize multiple attempts and choose which to
develop further. AIDE retains candidate programs in a solution tree and
refines promising solutions using summaries of previous attempts
\citep{jiang2025aide}. AIRA examines the interaction between generation
operators and greedy, evolutionary, and Monte Carlo tree-search policies
\citep{toledo2025aira}. MLEvolve combines graph-based search with
descendant-reward propagation, cross-branch references, and retrospective memory
\citep{du2026mlevolve}.
EvoTreeNAD uses top-percentile family values to select a retained root-to-node
lineage whose design history supplies the context for the next architectural
variation.

\paragraph{Evolutionary discovery with LLMs.}
Evolution through Large Models uses LLMs as program mutation operators,
linking pretrained generative capabilities to evolutionary optimization
\citep{lehman2024elm}. Eureka evolves reward programs for reinforcement
learning through best-of-\(N\) refinement, using the best-performing candidate
to inform the next generation \citep{ma2024eureka}.
FunSearch and AlphaEvolve demonstrate mathematical and algorithmic discovery
through LLM-generated programs, programmatic evaluation, and evolutionary
selection \citep{romera2024funsearch,novikov2025alphaevolve}.
ShinkaEvolve develops sample-efficient program evolution, while the Darwin
G{\"o}del Machine maintains an expanding archive of self-modifying agents
from which further variants are generated
\citep{lange2026shinkaevolve,zhang2026dgm}.
In EvoTreeNAD, successful designs can be developed further while subsequent
descendant outcomes can redirect evolution toward retained alternatives.

\paragraph{LLM-based neural architecture discovery.}
Conventional NAS, including evolutionary approaches, typically optimizes
architectures within a specified search space
\citep{elsken2018efficient,real2019regularized}.
LLM-based methods extend architecture generation and modification through
executable code. EvoPrompting combines code-level variation with evolutionary
search, and LLMatic integrates architectural variation with quality-diversity
optimization \citep{chen2023evoprompting,nasir2024llmatic}.
LeMo-NADe generates architectures under user-defined requirements, while
GENIUS, LM-Searcher, and RZ-NAS combine LLM proposals with structured
architecture optimization or efficient evaluation
\citep{rahman2024lemonademultiparameterneuralarchitecture,
zheng2023can,hu2025lm,ji2025rznas}.
NADER develops a supplied initial architecture through multi-agent
modification and reflection; RevoNAD refines base architectures through
reflective exploration and multi-objective evolutionary selection
\citep{yang2025nader,chang2025revonad}.
Starting from an empty root, EvoTreeNAD unifies the construction of complete
trainable architectures with their subsequent lineage-conditioned evolution
in a persistent genealogy, without requiring a supplied seed architecture
or a hand-specified architecture search space.


\section{EvoTreeNAD}
\label{sec:methods}

EvoTreeNAD organizes architecture discovery through a persistent genealogy of
evaluated architectures and their parent--child relationships
(Figure~\ref{fig:evotree_overview}). At each iteration, top-percentile family
values computed from the current genealogy guide the selection of a root-to-node
lineage for further evolution. Using architectures and evaluation outcomes from
the selected lineage, an Idea Agent proposes an architectural design or
modification, and a Code Agent implements it as a complete trainable architecture.
After evaluation, the resulting child is added to the genealogy, providing both
a candidate for further variation and new evidence for subsequent lineage selection.

\subsection{Task Setting}
\label{sec:interface}
\label{sec:evotree_overview}

A task interface \(\mathcal S\) fixes the model inputs and outputs,
supervised target, data splits, discovery evaluator, and reporting protocol,
while leaving the architecture's internal modules, connectivity, and forward
computation to be determined through evolution. Here, an architecture denotes
the complete trainable model specification, including its model-defined loss.
The task, data, evaluator, and stage-specific training protocols remain fixed
across candidates. Appendix~\ref{appendix:model_contract} provides the model
interface.

\begin{figure*}[t]
    \centering
    \includegraphics[width=\linewidth]{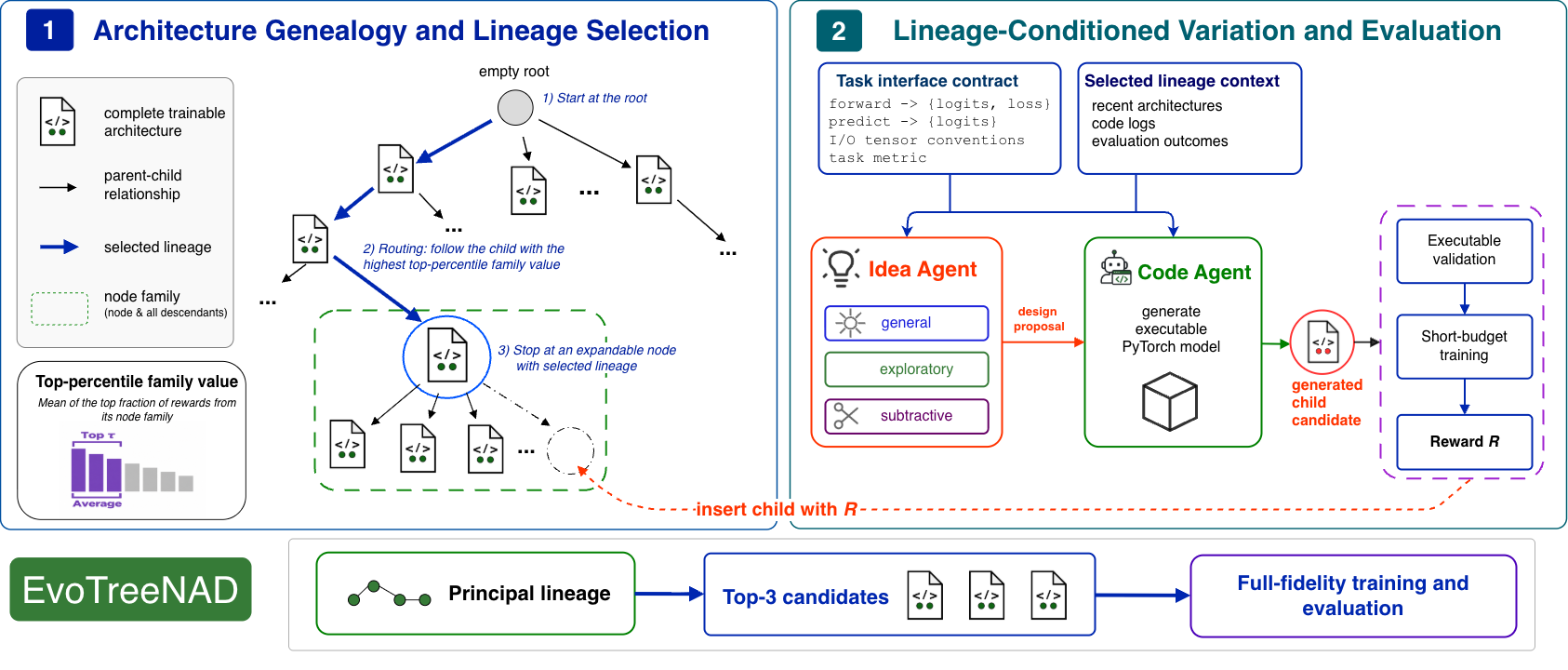}
    \caption{Overview of EvoTreeNAD.}
    \label{fig:evotree_overview}
\end{figure*}

\subsection{Architecture Genealogy and Family Values}
\label{sec:score_reward_value}
\label{sec:reward}
\label{sec:genealogical_formulation}

EvoTreeNAD begins with an empty root \(\rho\), which anchors routing but
does not represent an architecture. Each non-root node \(s\) represents an
architecture \(a_s\), and each edge records a parent-to-child design change.
After \(t\) evolution rounds, the genealogy is
\(\mathcal G_t=(\{\rho\}\cup\mathcal V_t,\mathcal E_t)\).

\paragraph{Node rewards.}
The discovery score \(S(s)\) measures the performance of architecture \(a_s\)
under a fixed short-budget training and validation protocol.
We convert this score into a comparable node reward while preserving the
metric direction:
\begin{equation}
R(s)=r(S(s))+\lambda_{\mathrm{size}}g(\mathrm{Params}(a_s)),
\qquad
r(v)=\sigma\!\left(\frac{\kappa d\,(v-v_e)}{|v_l-v_e|}\right),
\quad
d=\operatorname{sgn}(v_l-v_e).
\label{eq:reward_sigmoid_clean}
\end{equation}
Here \(\sigma\) is the logistic sigmoid, \(\kappa>0\) controls its slope,
\(v_e\) and \(v_l\) are fixed task-level reference values, and
\(g\) maps parameter count to a normalized compactness score in \([0,1]\).
The coefficient
\(\lambda_{\mathrm{size}}\) controls the compactness contribution. The node reward
\(R(s)\) evaluates one architecture; the family value defined below
aggregates node rewards over a descendant family to guide routing.

\paragraph{Top-percentile family value.}
\label{sec:value}
For a non-root node \(s\), its \emph{descendant family} contains \(s\) and
every architecture already evolved below it:
\begin{equation}
\mathcal D_t(s)=\{u\in\mathcal V_t:s\preceq_t u\},
\qquad
\mathcal H_t(s)=\biguplus_{u\in\mathcal D_t(s)}\{R(u)\}.
\label{eq:descendant_family_record}
\end{equation}
Here \(s\preceq_t u\) means that \(s\) is an ancestor of \(u\) in
\(\mathcal G_t\), including \(s=u\).

In our genealogy-based evolutionary framework, each non-root node is assigned
a value for routing, termed its \emph{family value}.
A family-value rule \(\Phi\) maps the rewards in \(\mathcal H_t(s)\)
to the value of node \(s\), \(V_t^\Phi(s)=\Phi(\mathcal H_t(s))\).
EvoTreeNAD uses a top-percentile rule to emphasize evidence from high-reward
family members while reducing sensitivity to weak variants.
Fix \(0<\tau\le1\) for the top-percentile rule \(\Phi_{\mathrm{top}}\).
Let \(n_t(s)=|\mathcal D_t(s)|\),
\(k_t(s)=\max\{1,\lceil\tau n_t(s)\rceil\}\), and let
\(x^{(s)}_{[1]}\ge\cdots\ge x^{(s)}_{[n_t(s)]}\) be the rewards in
\(\mathcal H_t(s)\). The top-percentile family value of node \(s\) is
\begin{equation}
V_t^{\mathrm{top}}(s)
=
\frac{1}{k_t(s)}
\sum_{i=1}^{k_t(s)}x^{(s)}_{[i]}.
\label{eq:value_family_top_tau}
\end{equation}

\subsection{Evolutionary Iteration}
\label{sec:evolution_iteration}

Each evolutionary iteration begins with lineage selection based on
top-percentile family values computed from the current genealogy.
Algorithm~\ref{alg:EvoTreeNAD} summarizes the complete evolutionary process.

\paragraph{Lineage selection.}
\label{sec:selection}
Each node \(s\) has fixed child capacity \(W(s)\); \(\operatorname{Ch}_t(s)\)
denotes its children in \(\mathcal G_t\). Routing starts at \(\rho\).
If the current node has remaining capacity, traversal stops. If it is
saturated, traversal follows its highest-valued child:
\begin{equation}
s_t^{(j+1)}
\in
\arg\max_{u\in\operatorname{Ch}_t(s_t^{(j)})}
V_t^{\mathrm{top}}(u),
\quad
\text{while }
|\operatorname{Ch}_t(s_t^{(j)})|\ge W(s_t^{(j)}).
\label{eq:allocation_traversal}
\end{equation}
The resulting root-to-node path
\(L_t=(\rho,s_t^{(1)},\ldots,s_t^{(m)})\) is the \emph{selected lineage}.
Its endpoint is the first node with remaining capacity and becomes the parent
of the next architecture. Because every iteration
begins again at the root and uses the current genealogy, later descendant
outcomes can redirect evolution toward a retained branch.

\paragraph{Architectural variation with Idea and Code Agents.}
\label{sec:variation}

The selected lineage provides the design and evaluation history for two agents
with distinct roles.
Agent prompts and context settings are provided in
Appendices~\ref{appendix:prompts} and~\ref{appendix:evotree_hparams}.
\begingroup
\setlength{\leftmargini}{1.5em}
\setlength{\labelsep}{0.4em}
\setlength{\labelwidth}{1.1em}
\begin{enumerate}
\setlength{\itemsep}{2pt}
\setlength{\parskip}{0pt}
\setlength{\parsep}{0pt}
\setlength{\topsep}{2pt}
\item \textbf{Idea Agent.}
The Idea Agent uses the task requirements and recent architectures and evaluation
outcomes from the selected lineage to propose distinct architectural designs or
modifications. These span \emph{general refinements}
of existing structures, \emph{exploratory modifications} that investigate
alternative designs, and \emph{subtractive simplifications} of low-contribution
components.

\item \textbf{Code Agent.}
The Code Agent implements a selected design proposal as a complete architecture
for evaluation, using the task interface and design history from the same
lineage.
\end{enumerate}
\endgroup

\noindent
\begin{minipage}[t]{0.485\textwidth}
\vspace{0pt}
\hrule height 0.8pt
\kern2pt
\captionsetup{type=algorithm}
\captionof{algorithm}{EvoTreeNAD genealogy-guided evolution}
\label{alg:EvoTreeNAD}
\par\kern2pt
\hrule height 0.4pt
\kern2pt
\footnotesize
\begin{algorithmic}[1]
\REQUIRE Task interface \(\mathcal S\), evaluator, top-percentile fraction
\(\tau\), child capacities \(W\), iterations \(H\)
\STATE Initialize \(\mathcal G_0\) with empty root \(\rho\)
\FOR{\(t=0,\ldots,H-1\)}
    \STATE Compute \(V_t^{\mathrm{top}}\) from \(\mathcal G_t\)
    \STATE \(s\leftarrow\rho\)
    \WHILE{\(|\operatorname{Ch}_t(s)|\ge W(s)\)}
        \STATE \(s\leftarrow
        \arg\max_{u\in\operatorname{Ch}_t(s)}
        V_t^{\mathrm{top}}(u)\)
    \ENDWHILE
    \STATE \(L_t\leftarrow\operatorname{Path}(\rho,s)\)
    \STATE Retrieve unused proposals \(I_t\) for \(s\); if none remain,
    set \(I_t\leftarrow\textsc{IdeaAgent}(\mathcal S,L_t)\)
    \STATE Select an unused proposal \(i_t\in I_t\) for \(s\)
    \STATE \(A_{t+1}\leftarrow\textsc{CodeAgent}(\mathcal S,L_t,i_t)\)
    \STATE Evaluate \(A_{t+1}\); compute \(R_{t+1}\)
    \STATE Add \(A_{t+1}\) with \(R_{t+1}\) under \(s\) to form \(\mathcal G_{t+1}\)
\ENDFOR
\STATE Extract the principal lineage from \(\mathcal G_H\)
\end{algorithmic}
\par\kern2pt
\hrule height 0.4pt
\end{minipage}\hfill
\begin{minipage}[t]{0.485\textwidth}
\vspace{0pt}
\paragraph{Architecture evaluation.}
\label{sec:evaluation}
Generated implementations are checked for executability and compatibility with
the task interface before short-budget training and validation under the fixed
task-specific protocol. Poorly performing candidates may be stopped early, with
recorded validation results still used to compute discovery scores and node
rewards. Attempts that yield no valid score trigger bounded Code-Agent retries
within the selected lineage. Appendix~\ref{app:discovery_evaluation} provides
details of these checks, early stopping, and failure handling.

\paragraph{Genealogy growth.}
After evaluation, the new architecture \(a_u=A_{t+1}\) is added as a child
of the selected parent together with its reward \(R(u)=R_{t+1}\), forming
\(\mathcal G_{t+1}\). Both improvements and weaker variants remain available
for further evolution. The new reward becomes part of the descendant-family
records, providing additional evidence for subsequent lineage selection.
\end{minipage}

\subsection{Final Architecture Selection}
\label{sec:principal_lineage}
After discovery, EvoTreeNAD extracts the \emph{principal lineage} by applying
the same routing rule (Eq.~\eqref{eq:allocation_traversal}) to the final
genealogy \(\mathcal G_H\). The three architectures on this lineage with the
highest discovery rewards form a fixed full-fidelity candidate set. Genealogy
growth, principal-lineage extraction, and candidate-set formation use only the
discovery training and validation data. After this set is fixed, its
architectures are trained from scratch under the full-fidelity reporting
protocol and evaluated on the official test split. Discovery and full-fidelity
protocols are provided in Appendices~\ref{appendix:discovery_protocol}
and~\ref{app:full_training_details}.

\FloatBarrier

\section{Dynamics of Genealogy-Guided Evolution}
\label{sec:stationary_selection_variation}

In EvoTreeNAD, accumulated descendant outcomes guide lineage selection,
which determines the design history used for subsequent variation.
We establish the existence of stationary empirical variation regimes for this
adaptive genealogy-growth process. Under specified assumptions on lineage-conditioned variation,
we further show that sustained top-percentile family values quantify the
probability of generating high-reward architectures in these regimes.

We consider EvoTreeNAD and its mean and maximum variants. Let \(\Phi\)
denote the corresponding family-value rule and \(\Omega_t^\Phi\) the process
state after \(t\) genealogy expansions, including the genealogy, observed
rewards, and retained information used for subsequent variation.
Rewards lie in a bounded interval \([a,b]\). One step of the induced process is
\begin{equation}
\begin{aligned}
L_t^\Phi &= \operatorname{Route}_{\Phi}(\Omega_t^\Phi),\\
(A_{t+1}^\Phi,R_{t+1}^\Phi)
&\sim K(\cdot\mid C_t^\Phi),\\
\Omega_{t+1}^\Phi
&\sim\operatorname{Grow}(\cdot\mid\Omega_t^\Phi,L_t^\Phi,
A_{t+1}^\Phi,R_{t+1}^\Phi),
\end{aligned}
\label{eq:main_genealogy_process}
\end{equation}
where \(C_t^\Phi\) is the variation context constructed from the process
state \(\Omega_t^\Phi\) and selected lineage \(L_t^\Phi\).
The kernel \(K\) specifies the conditional joint law of the next architecture and its reward, and
\(\operatorname{Grow}\) inserts the child and updates retained variation information.
Context construction and the variation and state-update kernels are shared
across rules, while the selected contexts and subsequent genealogies are
policy-dependent.

\paragraph{Stationary selection-variation.}
A conditional reward law \(F\) describes the possible rewards of variation
from a selected lineage context. We characterize ongoing variation through
the joint sequence of lineage-conditioned reward laws and realized rewards.
Empirical occupation measures over time shifts of this sequence admit
shift-stationary subsequential limits. We call these limits reachable
stationary regimes of the run.
\par\smallskip
\begin{theorem}[Stationary selection-variation]
\label{thm:main_stationary_selection_variation}
Every infinite run of Eq.~\eqref{eq:main_genealogy_process} admits reachable
stationary regimes. Almost surely, the following
conclusions hold simultaneously for every such regime \(\mathsf Q\). Let \(\nu_{\mathsf Q}\)
be the occupation distribution over lineage-conditioned reward laws and
\(\Lambda_{\mathsf Q}\) the reward distribution of generated architectures. Then
\begin{equation}
\Lambda_{\mathsf Q}=\int F\,\nu_{\mathsf Q}(dF).
\label{eq:main_stationary_mixture}
\end{equation}
Along every subsequence realizing \(\mathsf Q\), the empirical exceedance frequency at
each threshold \(q\) satisfying \(\Lambda_{\mathsf Q}(\{q\})=0\) converges to the
corresponding occupation-weighted conditional probability.
\end{theorem}
We next connect top-percentile family values to high-reward production within
these regimes and characterize how mean and maximum rules differ.

\paragraph{Family values and recurrent development.}
The variation assumptions in Appendix~\ref{app:genealogy_flow_production}
keep the high- and weak-reward distributions and intermediate reward \(r_0\)
fixed, while their mixture weights vary with the selected lineage context.
The highest-reward \(\tau\) fraction of each conditional reward law contains
all high outcomes and excludes weak ones. Let \(\mu_+\) denote the mean high reward.
The mean \(U_\tau(F)\) of the highest-reward \(\tau\) fraction of \(F\) is
the population counterpart of the empirical top-percentile family value.

Let \(q\) separate high rewards from intermediate and weak outcomes.
A family is \emph{recurrent} if selected lineages pass through its node
infinitely often.
For each recurrent family \(v\), we assume that the conditional high- and
weak-outcome probabilities have convergent averages along its sequence of
selected contexts. Let \(\rho_{\Phi,v}(q)\) denote the limiting average
conditional probability of reward at least \(q\). At the root, this average
covers the complete run.
\par\smallskip
\begin{theorem}[Top-percentile routing and recurrent high-reward production]
\label{thm:main_genealogy_flow_production}
Under the lineage-conditioned variation assumptions in
Appendix~\ref{app:genealogy_flow_production}, the following comparisons hold
almost surely in the genealogy generated by each rule.
\par\noindent
\textbf{Top-percentile.} For every fixed recurrent non-root family \(v\)
and every reachable stationary regime \(\mathsf Q\) of the same run,
\begin{equation}
\lim_{t\to\infty}V_t^{\mathrm{top}}(v)
=
U_\tau(\Lambda_{\mathsf Q})
=
r_0+\frac{\mu_+-r_0}{\tau}\Lambda_{\mathsf Q}([q,b]).
\label{eq:main_top_stationary_value}
\end{equation}
The limiting family value therefore characterizes the regime's top-percentile
reward mean and high-reward production probability. When \(\Phi\) is the
top-percentile rule, the equality \(\rho_{\Phi,v}(q)=\Lambda_{\mathsf Q}([q,b])\)
holds for every recurrent family \(v\) in the same run, connecting the high-reward
production rate of each recurrent family to that of the run.
\par\noindent
\textbf{Full-family mean.} Weak-outcome frequency reduces the limiting
family value. Recurrent child families attain equal limiting means but can
have different high-reward frequencies. The gap between the maximal high-reward
frequency among recurrent children and the parent frequency is proportional
to the corresponding gap in weak-outcome frequency
(Eq.~\eqref{eq:app_mean_production_gap}).
\par\noindent
\textbf{Family maximum.} The resulting genealogy growth is pathwise equivalent
to best-of-\(N\) greedy continuation. The current parent \(v\) produces
\(N=W(v)\) children, the child with the highest immediate reward becomes the
next parent, and unchosen siblings receive no further development.
\par\noindent
For each rule,
\(\rho_{\Phi,\rho}(q)\) equals the realized long-run frequency of generated
architectures with reward at least \(q\).
\end{theorem}

Maintaining a high top-percentile family value as a family grows requires a
nonvanishing fraction of high-reward descendants, rather than a few early
successes (Proposition~\ref{prop:top_fraction_evidence}).
Under mean routing, a recurrent child family can produce high-reward
architectures more frequently than a recurrent sibling yet receive no higher
limiting value because it also produces weak variants more frequently.
The limiting top-percentile family value instead tracks high-reward production
without this weak-outcome penalty. Unlike the maximum-family variant,
EvoTreeNAD can develop successful architectures without making earlier branch
choices permanent, allowing subsequent outcomes to redirect evolution toward
retained alternatives.
Full statements and proofs appear in Appendix~\ref{appendix:routing_theory}.


\section{Experiments}
\label{sec:experiments}

We evaluate EvoTreeNAD on CIFAR-10/100 and six MedMNIST-v2 tasks spanning
diverse medical imaging modalities and both 2D image and 3D volume
classification. We report the performance of the discovered architectures and
use controlled CIFAR-10 comparisons to assess the benefit of genealogy-guided
evolution.

\subsection{Protocol}

CIFAR-10/100~\citep{krizhevsky2009learning} provide established architecture
discovery benchmarks. We split each 50k training set into 42k/8k
discovery-training and validation subsets. Final performance is evaluated on
the official 10k test set. MedMNIST-v2~\citep{medmnistv1,medmnistv2} uses its official
train, validation, and test splits
(dataset details in Appendix~\ref{appendix:medmnist_description_statistics}).
EvoTreeNAD begins from an empty root on every task.
Main discovery uses three independent runs per task, with GPT-4.1 for architectural
proposals and gpt-oss-20b (OSS20B) for code realization.
We mask dataset identifiers in agent prompts to reduce LLM reliance on
benchmark-specific prior knowledge. CIFAR-10/100 architecture
discovery uses budgets of up to 300 iterations, with the
corresponding discovery costs reported in Table~\ref{tab:performance_cifar}.
CIFAR-10 discovery-strategy and agent-configuration comparisons use 100
iterations per run. MedMNIST-v2 discovery runs use 80 iterations each.

Candidates receive short-budget evaluation during evolution. After each
EvoTreeNAD run, the three architectures with the highest discovery rewards on
its principal lineage form a fixed candidate set for full-fidelity evaluation.
Genealogy growth and candidate-set formation use only
discovery training and validation data. Each candidate is then trained from
scratch under the fixed task-specific full-fidelity recipe, with no pretraining,
external data, or architecture-specific tuning after selection. Reported task
metrics use the official test splits. Discovery budgets, hyperparameters, and
training recipes appear in Appendices~\ref{appendix:discovery_protocol},
\ref{appendix:evotree_hparams}, and~\ref{app:full_training_details}.
Code will be released on GitHub.

\subsection{Architecture Discovery on CIFAR-10/100}
\begin{table*}[htbp]
  \caption{CIFAR-10/100 test performance. For each dataset, we report the best-performing discovered architecture and a smaller competitive architecture from a different discovery run. EvoTreeNAD architectures are trained from scratch under fixed full-fidelity recipes; mean$\pm$std is computed over four random seeds. Results for other methods are taken from their publications.}
  \label{tab:performance_cifar}
  \begin{center}
    \scriptsize
    \setlength{\tabcolsep}{2.8pt}
        \begin{tabular}{@{}p{4.4cm}ccc@{\hspace{5pt}}ccc@{\hspace{5pt}}cc@{}}
          \toprule
          \multirow{2}{*}{Approach}
            & \multicolumn{3}{c}{CIFAR-10}
            & \multicolumn{3}{c}{CIFAR-100}
            & \multirow{2}{*}{Method}
            & \multirow{2}{*}{Space} \\
            \cmidrule(lr){2-4} \cmidrule(lr){5-7}
            & Top-1 err. (\%)$\downarrow$ & Params (M) & GD
            & Top-1 err. (\%)$\downarrow$ & Params (M) & GD
            & & \\
          \midrule
          \multicolumn{9}{@{}l}{\textit{Classical NAS and differentiable NAS}} \\
          NASNet-A \citep{zoph2018learningtransferablearchitecturesscalable} & 2.65 & 3.3 & 2000 & N/R & N/R & N/R & RL & NASNet \\
          ENAS \citep{pham2018efficientneuralarchitecturesearch} & 2.89 & 4.6 & 0.45 & N/R & N/R & N/R & RL & NASNet \\
          AmoebaNet-B \citep{real2019regularized} & 2.55\std{0.05} & 2.8 & 3150 & N/R & N/R & N/R & EA & NASNet \\
          Random Search \citep{li2019randomsearchreproducibilityneural} & 2.85\std{0.08} & 4.3 & 9.7 & N/R & N/R & N/R & Random & DARTS \\
          DARTS \citep{liu2018darts} & 2.76\std{0.09} & 3.3 & 4 & N/R & N/R & N/R & GB & DARTS \\
          P-DARTS \citep{chen2021progressive} & 2.50 & 3.4 & 0.3 & 15.92 & 3.6 & 0.3 & GB & DARTS \\
          $\beta$-DARTS \citep{ye2022b} & 2.51\std{0.07} & 3.78 & 0.4 & 16.24\std{0.22} & 3.8 & 0.4 & GB & DARTS \\
          $\Lambda$-DARTS \citep{movahedi2022lambda} & 2.35 & 3.8 & 0.8 & 15.8 & 3.8 & 0.8 & GB & DARTS \\
          EG-NAS \citep{Cai_Chen_Liu_Ling_Lai_2024} & 2.53 & 3.2 & 0.1 & 16.22 & 3.2 & 0.1 & GB & DARTS \\
          ProxylessNAS-G \citep{cai2019proxylessnasdirectneuralarchitecture} & 2.08 & 5.7 & N/R & N/R & N/R & N/R & GB & DARTS \\
          \midrule
          \multicolumn{9}{@{}l}{\textit{LLM-guided NAS and open-ended NAD}} \\
          EvoPrompting \citep{chen2023evoprompting} & 6.89\std{0.90} & N/R & N/R & 29.61\std{0.58} & N/R & N/R & LLM & NATS \\
          GENIUS \citep{zheng2023can}$^{\ddagger}$ & 6.21\std{0.09} & N/R & N/R & 29.09\std{0.72} & N/R & N/R & LLM & NB201 \\
          LM-Searcher \citep{hu2025lm} & 3.10 & N/R & N/R & 27.04 & N/R & N/R & LLM & DARTS \\
          RZ-NAS \citep{ji2025rznas} & 2.41\std{0.13} & N/R & 0.03 & 17.49\std{0.08} & N/R & 0.03 & LLM & DARTS \\
          RevoNAD \citep{chang2025revonad}$^{\ddagger}$ & 4.78\std{0.2} & N/R & N/R & 23.62\std{0.32} & N/R & N/R & LLM & None \\
          NADER \citep{yang2025nader}$^{\ddagger}$ & 5.38 & N/R & N/R & 24.00 & N/R & N/R & LLM & None \\
          \midrule
          \multicolumn{9}{@{}l}{\textbf{EvoTreeNAD (ours) with no predefined architecture space or supplied seed architecture}} \\
          \rowcolor{OursBlue}
          Best discovered architecture & \textbf{2.05\std{0.06}} & 7.03 & 0.63 & \textbf{15.09\std{0.22}} & 7.67 & 1.02 & \textbf{LLM} & \textbf{None} \\
          \rowcolor{OursBlue}
          Smaller discovered architecture & 2.38\std{0.06} & \textbf{2.25} & \textbf{0.38} & 15.18\std{0.10} & 5.66 & 1.14 & \textbf{LLM} & \textbf{None} \\
          \rowcolor{OursBlue!55}
          \textit{Best training-seed result} & 2.00 & N/R & N/R & 14.82 & N/R & N/R & LLM & None \\
          \bottomrule
        \end{tabular}
    \par\smallskip
    {\raggedright
    GD denotes reported discovery GPU-days. For EvoTreeNAD, GD reports the
    discovery cost of the corresponding run. Space denotes the reported
    hand-specified search space, if any. N/R indicates not reported.
    $^{\ddagger}$ marks the official NAS-Bench-201 protocol.\par}
  \end{center}
  \vskip -0.1in
\end{table*}

EvoTreeNAD discovers architectures that outperform the listed NAS and NAD
methods, achieving test errors of \(2.05{\pm}0.06\%\) on CIFAR-10 and
\(15.09{\pm}0.22\%\) on CIFAR-100 (Table~\ref{tab:performance_cifar}).
On CIFAR-10, it also discovers a 2.25M-parameter architecture with
\(2.38{\pm}0.06\%\) error, improving on DARTS and P-DARTS with fewer parameters.
On CIFAR-100, both discovered architectures surpass the reported results of
all listed baselines. These results demonstrate that EvoTreeNAD can discover
multiple high-performing architectures that are competitive with established
NAS and NAD methods
(representative architectures in Appendix~\ref{appendix:representative_architectures}).

\subsection{Comparison of Discovery Strategies}

To assess the benefit of genealogy-guided evolution, we compare EvoTreeNAD
with repeated direct generation, best-of-\(N\) greedy continuation, and
full-family-mean routing on CIFAR-10
(strategy schematics in Appendix Figure~\ref{fig:ablation_schematic}).
Each method uses three independent
100-iteration runs with OSS20B for both agents. Repeated direct generation produces
independent architectures from the empty root. The maximum-family variant is
exactly best-of-\(N\) greedy continuation
(Proposition~\ref{prop:app_maximum_greedy}), generating siblings and continuing
only from the child with the highest immediate reward. The full-family mean variant replaces
the top-percentile family value of each node with the mean reward of its entire
family. Both variants change the family-value
rule within EvoTreeNAD's lineage-conditioned generation and evaluation pipeline.

\begin{table*}[t]
\centering
\caption{CIFAR-10 comparison of architecture-discovery strategies.
Each strategy has three independent 100-iteration runs. The Run 1--3 columns report the highest full-fidelity accuracy (\%) among three candidates selected by discovery reward within each run; Best and Mean summarize these three run-level accuracies.}
\label{tab:discovery_strategy_controls}
\scriptsize
\setlength{\tabcolsep}{3.4pt}
\resizebox{\textwidth}{!}{%
\begin{tabular}{lccc ccc cc}
\toprule
Strategy
& \shortstack{Lineage-conditioned\\variation}
& \shortstack{Branch\\reconsideration}
& Selection signal
& Run 1 & Run 2 & Run 3 & Best & Mean \\
\midrule
Repeated direct generation
& -- & -- & --
& 96.94 & 96.59 & 96.60 & 96.94 & 96.71 \\
Best-of-\(N\) greedy continuation
& \checkmark & -- & Immediate child reward
& 96.34 & 96.92 & 97.06 & 97.06 & 96.77 \\
Full-family mean
& \checkmark & \checkmark & Mean-based family value
& 96.81 & 95.66 & 96.49 & 96.81 & 96.32 \\
\rowcolor{OursBlue}
\textbf{EvoTreeNAD}
& \checkmark & \checkmark & \textbf{Top-percentile family value}
& \textbf{97.39} & \textbf{97.26} & \textbf{97.11}
& \textbf{97.39} & \textbf{97.25} \\
\bottomrule
\end{tabular}%
}
\end{table*}

EvoTreeNAD consistently outperforms direct generation, best-of-\(N\) greedy
continuation, and full-family-mean routing. Each of its three runs achieves
higher accuracy than the best run of any alternative
(Table~\ref{tab:discovery_strategy_controls}). The advantage over
full-family-mean routing shows the benefit of top-percentile selection over
averaging all family rewards within the same genealogy-guided framework.
The advantage over mean routing and greedy
continuation is consistent with the distinctions established in
Theorem~\ref{thm:main_genealogy_flow_production}, which characterizes the
sensitivity of full-family means to weak descendants and the irreversible
lineage choices of maximum-family routing.

\subsection{Architecture Evolution and Agent Configurations}

\begin{figure*}[t]
    \centering
    \captionsetup[subfigure]{justification=raggedright,singlelinecheck=false}
    \begin{subfigure}[c]{0.45\textwidth}
        \centering
        \includegraphics[width=\linewidth]{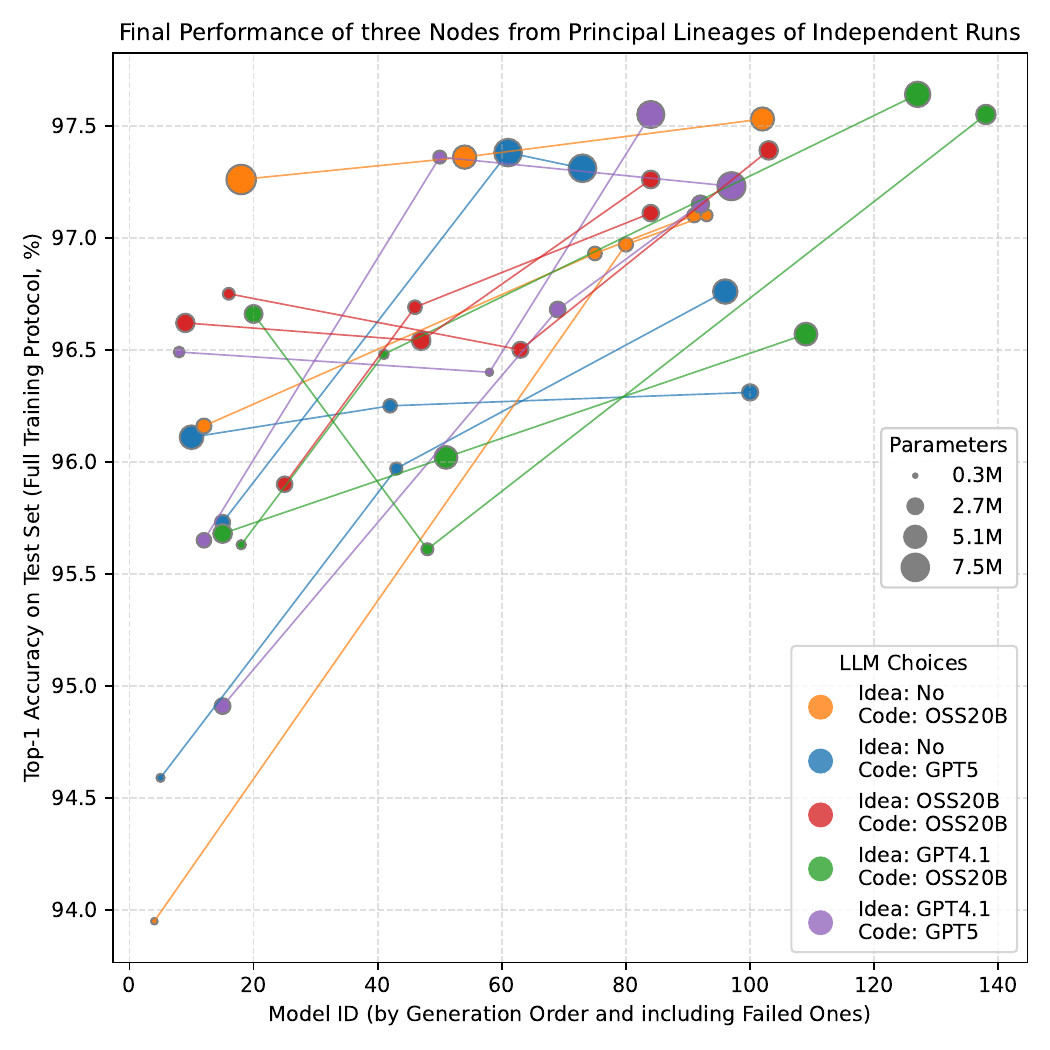}
        \caption{Full-fidelity test accuracy of three architectures spanning the early and later stages of each principal lineage from independent 100-iteration runs.}
        \label{fig:main_lineage_progression}
    \end{subfigure}
    \hfill
    \begin{minipage}[c]{0.43\textwidth}
        \begin{subfigure}[t]{\linewidth}
            \centering
            \includegraphics[width=\linewidth]{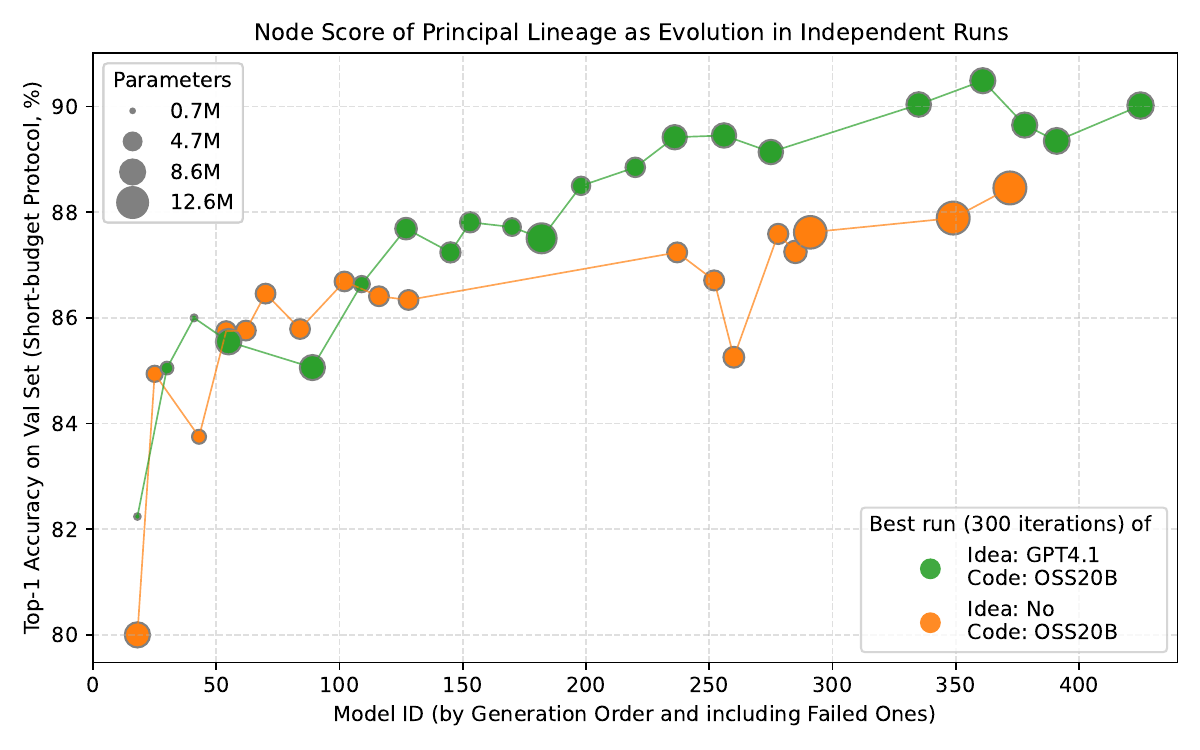}
            \caption{Short-budget validation accuracy along principal lineages from two 300-iteration runs.}
            \label{fig:main_extended_proxy}
        \end{subfigure}
        \vspace{1pt}
        \begin{subfigure}[t]{\linewidth}
            \centering
            \includegraphics[width=\linewidth]{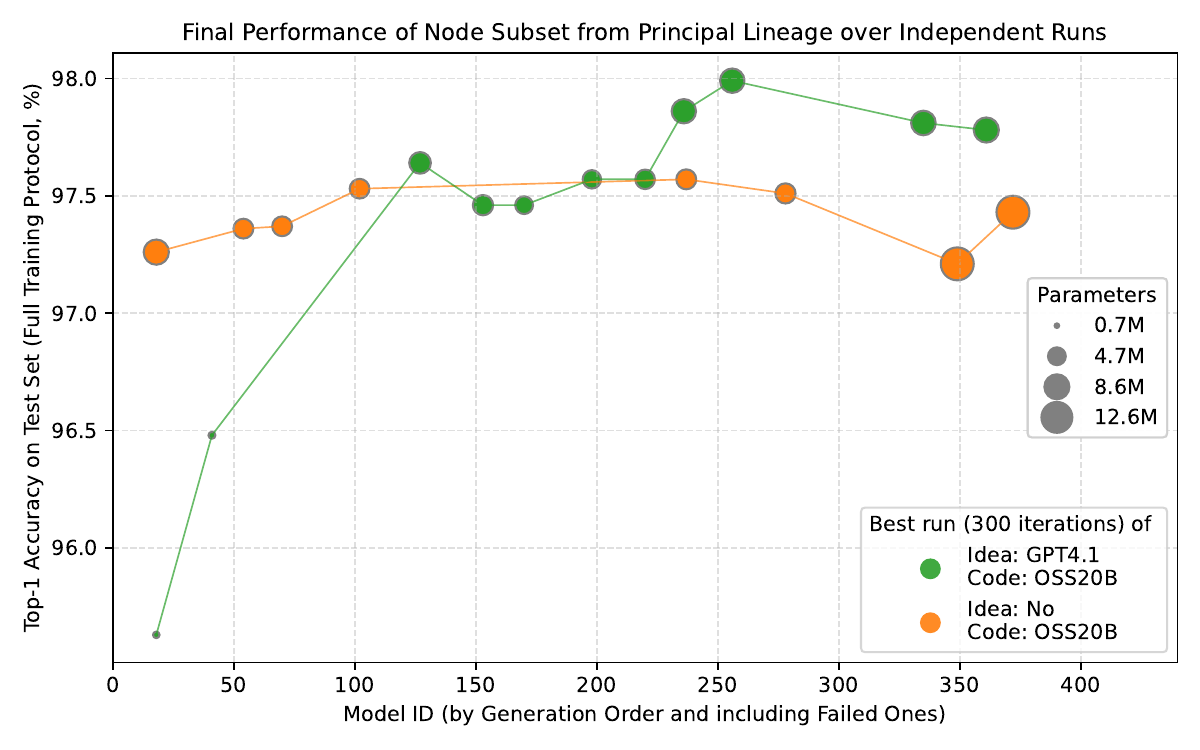}
            \caption{Full-fidelity test accuracy of architectures from the same principal lineages.}
            \label{fig:main_extended_full}
        \end{subfigure}
    \end{minipage}
    \caption{Architecture development along EvoTreeNAD principal lineages on CIFAR-10.}
    \label{fig:main_lineage_dynamics}
\end{figure*}

\begin{table}[t]
    \centering
    \setlength{\parskip}{0pt}
    \renewcommand{\std}[1]{{\color{StdBlueGray}\fontsize{4}{4.8}\selectfont\textpm#1}}
    \begin{minipage}[t]{0.44\textwidth}
        \vspace{0pt}
        \centering
        \captionsetup{font=footnotesize,aboveskip=2pt,belowskip=3pt}
        \caption{MedMNIST-v2 2D results.}
        \label{tab:medmnist_2d}
        \fontsize{6}{7.5}\selectfont
        \setlength{\tabcolsep}{1.5pt}
        \begin{tabular}{@{}lccc@{}}
            \toprule
            \textbf{Method} & \textbf{PathMNIST} & \textbf{OCTMNIST} & \textbf{TissueMNIST} \\
            \midrule
            ResNet-18 (28)\textsuperscript{\dag} & 0.907 & 0.743 & 0.676 \\
            ResNet-18 (224)\textsuperscript{\dag} & 0.909 & 0.763 & 0.681 \\
            ResNet-50 (28)\textsuperscript{\dag} & 0.911 & 0.762 & 0.680 \\
            ResNet-50 (224)\textsuperscript{\dag} & 0.892 & 0.776 & 0.680 \\
            AutoML (AutoKeras)\textsuperscript{\dag} & 0.834 & 0.763 & 0.703 \\
            Google AutoML Vision\textsuperscript{\dag} & 0.728 & 0.771 & 0.673 \\
            MSTF-NAS & 0.910 & 0.780 & 0.740 \\
            ZO-DARTS++ L & 0.804\std{0.024} & 0.799\std{0.010} & 0.656\std{0.012} \\
            EvoNASMed & 0.920 & 0.800 & 0.760 \\
            \midrule
            \rowcolor{OursBlue}
            \textbf{EvoTreeNAD (64)} & \textbf{0.962\std{0.003}} & \textbf{0.926\std{0.009}} & \textbf{0.769\std{0.001}} \\
            \bottomrule
        \end{tabular}
    \end{minipage}\hfill
    \begin{minipage}[t]{0.54\textwidth}
        \vspace{0pt}
        \centering
        \captionsetup{font=footnotesize,aboveskip=2pt,belowskip=3pt}
        \caption{MedMNIST-v2 3D results.}
        \label{tab:medmnist_3d}
        \fontsize{6}{7.5}\selectfont
        \setlength{\tabcolsep}{1.5pt}
        \begin{tabular}{@{}lccc@{}}
            \toprule
            \textbf{Method} & \textbf{Vessel3D} & \textbf{Synapse3D} & \textbf{Organ3D} \\
                \textit{Metric} & \textit{AUC, ACC} &  \textit{AUC, ACC} &  \textit{ACC} \\
            \midrule
            ResNet-18 + 3D\textsuperscript{\dag} & 0.874, 0.877 & 0.820, 0.745 & 0.907 \\
            ResNet-18 + ACS\textsuperscript{\dag} & 0.930, 0.928 & 0.705, 0.722 & 0.900 \\
            ResNet-50 + 3D\textsuperscript{\dag} & 0.907, 0.918 & 0.851, 0.795 & 0.883 \\
            ResNet-50 + ACS\textsuperscript{\dag} & 0.912, 0.858 & 0.719, 0.709 & 0.889 \\
            auto-sklearn\textsuperscript{\dag} & 0.910, 0.915 & 0.631, 0.730 & 0.814 \\
            AutoML (AutoKeras)\textsuperscript{\dag} & 0.773, 0.894 & 0.538, 0.724 & 0.804 \\
            EGNAS & 0.918\std{0.002}, 0.901\std{0.004} & 0.892\std{0.004}, 0.868\std{0.004} & 0.943\std{0.003} \\
            EvoNASMed & 0.940, 0.940 & 0.820, 0.846 &0.908 \\
            \midrule
            \rowcolor{OursBlue}
            \textbf{EvoTreeNAD ($64^3$)} & \textbf{0.975\std{0.008}, 0.970\std{0.002}} & \textbf{0.981\std{0.002}, 0.952\std{0.003}} & \textbf{0.956\std{0.004}} \\
            \bottomrule
        \end{tabular}
    \end{minipage}
    \par\smallskip
    {\fontsize{7}{8.5}\selectfont\raggedright
    The 2D tasks and Organ3D report accuracy; Vessel3D and Synapse3D report
    AUC, accuracy. EvoTreeNAD results report mean$\pm$std over four random seeds.
    \textsuperscript{\dag} marks official MedMNIST-v2 baselines
    from \cite{medmnistv2}. Other prior methods are MSTF-NAS~\citep{wang2024mednas},
    ZO-DARTS++~\citep{xie2025zo}, EvoNASMed~\citep{ali2024evolutionary},
    and EGNAS~\citep{benmeziane2025efficient}.\par}
\end{table}

\paragraph{Architecture outcomes along principal lineages.}
To examine how architectures develop along a lineage, we evaluate three
architectures spanning its early and later stages under the full-fidelity
protocol (Figure~\ref{fig:main_lineage_progression}). Later architectures often
outperform the initial designs despite intermediate performance declines.
For these CIFAR-10 architectures, short-budget validation scores correlate
positively with full-fidelity performance (Appendix Figure~\ref{fig:short_vs_full}),
supporting the use of short-budget scores to guide evolution.
The two 300-iteration runs provide a longer view of this
development, showing further gains in short-budget validation and full-fidelity
performance as evolution proceeds (Figures~\ref{fig:main_extended_proxy}
and~\ref{fig:main_extended_full}). In the run with GPT-4.1 proposals, later
evolution produces our best-performing CIFAR-10 architecture.
Architecture-selection protocols and details of the 300-iteration runs appear
in Appendices~\ref{appendix:evo_progress_ana} and~\ref{sec:extend_iter}.

\paragraph{Idea and Code Agent controls.}
We evaluate EvoTreeNAD under five Idea and Code Agent configurations,
including code-only variation. Each configuration uses three independent runs
with genealogy-guided selection held fixed. EvoTreeNAD discovers competitive
architectures across these configurations
(Figure~\ref{fig:main_lineage_progression}). The best-performing architectures
are discovered using GPT-4.1 as the Idea Agent and OSS20B as the Code Agent.
Code realization accounts for most of the agent token usage in these runs.
Locally deployed OSS20B handles this token-intensive role without API charges.
Appendix Figure~\ref{fig:appendix_agent_configuration} summarizes the
best architecture performance from each run, and Appendix
Table~\ref{tab:agent_comparison} reports process statistics and costs.

\subsection{Architecture Discovery across MedMNIST-v2 Tasks}

EvoTreeNAD discovers architectures that outperform the strongest listed
baseline on all six MedMNIST-v2 tasks
(Tables~\ref{tab:medmnist_2d} and~\ref{tab:medmnist_3d}). These gains cover
diverse medical imaging modalities and both 2D image and 3D volume
classification. Starting from an empty root on each task, the same
genealogy-guided framework produces high-performing architectures across these
settings, demonstrating its effectiveness beyond natural-image benchmarks.

\FloatBarrier

\section{Discussion and Conclusion}

EvoTreeNAD combines lineage-conditioned variation with selection guided by
top-percentile family values in a persistent architecture genealogy. Successful designs can support
further evolution while accumulated descendant evidence can redirect development
toward retained alternatives. Our two theorems give this process a quantitative
foundation by connecting continued genealogy growth, stationary variation
regimes, and high-reward production. Experiments demonstrate competitive
architecture performance on natural-image and medical-image classification
tasks, while controlled comparisons show the advantage of EvoTreeNAD over
direct generation, greedy continuation, and mean-based routing.

Scaling EvoTreeNAD depends on obtaining informative evaluation feedback at an
affordable computational cost. Overall, EvoTreeNAD demonstrates how LLM-driven
discovery can build on
prior designs while using accumulated experimental evidence to reconsider its
direction.

\subsection*{AI use statement}
The authors conceived and carried out this research and wrote the manuscript,
using LLM-based tools (ChatGPT and OpenAI Codex) for assistance.
The authors used LLMs to help improve the
literature review by identifying additional relevant references and refining
its organization. Within the theoretical framework developed by the authors,
LLMs assisted the authors in developing proofs, exploring proof strategies,
and refining and checking mathematical derivations. LLMs also assisted with drafting
and revising portions of the manuscript, including the theoretical section and
appendix. The authors reviewed all AI-assisted content, checked the mathematical
arguments and cited sources, and take responsibility for the entire work.

\bibliographystyle{evotreenad}
\bibliography{references}

\clearpage
\appendix

\section{Dynamics of Genealogy-Guided Evolution}
\label{appendix:routing_theory}

\setcounter{theorem}{0}
\renewcommand{\theHtheorem}{appendix.\arabic{theorem}}

This appendix provides the full statements and proofs of the theoretical results
in Section~\ref{sec:stationary_selection_variation}. We first formalize genealogy
growth in EvoTreeNAD and establish the properties of the family-value rules.
We then prove the existence of stationary empirical variation regimes and show
how sustained top-percentile family values characterize high-reward production,
alongside comparisons with mean and maximum routing.
Table~\ref{tab:theory_notation} lists the notation and corresponding descriptions.

\begin{table}[!ht]
\centering
\caption{Notation Table}
\label{tab:theory_notation}
\small
\setlength{\tabcolsep}{5pt}
\renewcommand{\arraystretch}{1.12}
\begin{tabular}{@{}p{0.30\linewidth}p{0.66\linewidth}@{}}
\toprule
Notation & Description \\
\midrule
\(\Phi\) &
Rule mapping a descendant-family reward multiset to a family value. \\
\(\Omega_t^\Phi,\ \mathcal G_t^\Phi\) &
Process state and its genealogy after \(t\) expansions under rule \(\Phi\). \\
\(\mathcal H_t^\Phi(v)\) &
Reward multiset of node \(v\) and all its descendants in the current genealogy. \\
\(V_t^\Phi(v)\) &
Empirical family value \(\Phi(\mathcal H_t^\Phi(v))\) used for routing. \\
\(\tau,\ V_t^{\mathrm{top}}(v)\) &
Fixed top-percentile fraction and corresponding empirical family value. \\
\(L_t^\Phi\) &
Selected root-to-node lineage; its endpoint receives the next child. \\
\(C_t^\Phi\) &
Variation context constructed from the process state and selected lineage. \\
\(K\) &
Lineage-conditioned kernel generating an architecture and its reward. \\
\(A_{t+1}^\Phi,\ R_{t+1}^\Phi\) &
Generated architecture and its realized reward at expansion \(t+1\) under rule \(\Phi\). \\
\(F_t^\Phi\) &
Conditional reward law of the next architecture, given the pre-iteration history. \\
\(\mathsf Q_{\Phi,T},\ \mathsf Q\) &
Empirical occupation measure over time shifts of the sequence of conditional reward laws and realized rewards, and a reachable stationary regime given by a shift-invariant weak subsequential limit. \\
\(\nu_{\mathsf Q}\) &
Marginal distribution of conditional reward laws under \(\mathsf Q\). \\
\(\Lambda_{\mathsf Q}\) &
Marginal distribution of generated rewards under \(\mathsf Q\). \\
\(M_{k,n}(x)\) &
Mean of the \(k\) largest entries of \(x\in[a,b]^n\). \\
\(U_\alpha(F)\) &
Mean of the highest-reward \(\alpha\) fraction of \(F\), for \(0<\alpha\le1\). For \(\alpha=\tau\), the population counterpart of the empirical top-percentile family value. \\
\(N_T^\Phi(v)\) &
Number of iterations \(t<T\) whose selected lineage contains \(v\). \\
\(\mathcal I_\Phi,\ \mathcal R_\Phi(v)\) &
Nodes visited infinitely often by selected lineages, and children of \(v\) with this property, respectively. \\
\(\rho_{\Phi,v}(q),\ \overline\ell_{\Phi,v}\) &
Limiting averages of conditional high- and weak-reward probabilities over iterations routed through \(v\). \\
\bottomrule
\end{tabular}
\end{table}

\subsection{Formal Setup}
\label{app:rooted_genealogy_process}

We consider EvoTreeNAD with a family-value rule \(\Phi\) and denote the
resulting stochastic process by \((\Omega_t^\Phi)_{t\ge0}\).
The process state \(\Omega_t^\Phi\) includes the genealogy \(\mathcal G_t^\Phi\)
after \(t\) expansions and the information retained for subsequent variation.
The initial genealogy consists only of the empty root \(\rho\).
Each non-root node \(v\) represents an architecture with reward
\(R^\Phi(v)\in[a,b]\), and every node has a fixed child capacity
\(W(v)\in\mathbb N_{+}\).

For a non-root node \(v\), its descendant-family reward multiset is
\begin{equation}
\mathcal H_t^\Phi(v)
=
\biguplus_{u\in\{v\}\cup\operatorname{Desc}_t^\Phi(v)}\{R^\Phi(u)\},
\label{eq:app_family_record}
\end{equation}
where \(\operatorname{Desc}_t^\Phi(v)\) denotes the set of strict descendants
of \(v\) in \(\mathcal G_t^\Phi\).
The family-value rule \(\Phi\) assigns \(v\) the value
\(V_t^\Phi(v)=\Phi(\mathcal H_t^\Phi(v))\) used for routing.
For a nonempty reward multiset \(\mathcal H=\{r_1,\ldots,r_n\}\) with rewards
ordered as \(r_{[1]}\ge\cdots\ge r_{[n]}\), set
\(k_\tau(n)=\max\{1,\lceil\tau n\rceil\}\) for \(0<\tau\le1\) and define
\begin{equation}
\Phi_{\max}(\mathcal H)=r_{[1]},
\qquad
\Phi_{\mathrm{mean}}(\mathcal H)=\frac1n\sum_{i=1}^n r_i,
\qquad
\Phi_{\mathrm{top}}(\mathcal H)
=\frac1{k_\tau(n)}\sum_{i=1}^{k_\tau(n)}r_{[i]}.
\label{eq:app_family_rules}
\end{equation}
EvoTreeNAD uses \(\Phi_{\mathrm{top}}\), while \(\Phi_{\mathrm{mean}}\) and
\(\Phi_{\max}\) define its mean and maximum variants.
We use \(\mathrm{top}\), \(\mathrm{mean}\), and \(\max\) for the
corresponding family-value rules.

Routing starts at \(\rho\) and stops at the first node with remaining child
capacity. At each saturated node \(v\), the next node is
\begin{equation}
\min_{\prec}\operatorname*{arg\,max}_{c\in\operatorname{Ch}_t^\Phi(v)}
V_t^\Phi(c),
\label{eq:app_rooted_routing}
\end{equation}
where \(\operatorname{Ch}_t^\Phi(v)\) denotes the children of \(v\) and
\(\prec\) is a fixed order for resolving ties. Let \(L_t^\Phi\) denote the
resulting root-to-node lineage.

The variation context \(C_t^\Phi\) is constructed from the process state
\(\Omega_t^\Phi\) and selected lineage \(L_t^\Phi\). It determines the
conditional law of the next architecture and reward. The process evolves
according to
\begin{equation}
\begin{aligned}
(A_{t+1}^\Phi,R_{t+1}^\Phi)
&\sim K(\,\cdot\mid C_t^\Phi),\\
\Omega_{t+1}^\Phi
&\sim\operatorname{Grow}(\cdot\mid\Omega_t^\Phi,L_t^\Phi,
A_{t+1}^\Phi,R_{t+1}^\Phi),
\end{aligned}
\label{eq:app_genealogy_transition}
\end{equation}
where \(A_{t+1}^\Phi\) and \(R_{t+1}^\Phi\) denote the generated
architecture and its reward, respectively. The state-update kernel
\(\operatorname{Grow}\) adds the generated architecture as a child of the
selected endpoint and updates the retained variation information.
All three policies use the same context construction, variation and state-update
kernels, capacity rule, and tie rule. Each policy's routing decisions determine the selected
lineages and their variation contexts, thereby shaping the architectures
generated and the genealogy that develops.

\subsection{Properties of Family-Value Rules}
\label{app:top_fraction_evidence}

Proposition~\ref{prop:app_maximum_greedy} establishes that EvoTreeNAD with
\(\Phi=\Phi_{\max}\) is pathwise equivalent to best-of-\(N\) greedy
continuation. Proposition~\ref{prop:top_fraction_evidence} quantifies the
sensitivity of top-percentile family values to individual rewards and shows
that maintaining a high value as a family grows requires a nonvanishing
fraction of high-reward descendants.

In best-of-\(N\) greedy continuation, the current parent \(v\) generates
\(N=W(v)\) children in succession. The child with the highest immediate reward
becomes the next parent, and the process repeats without returning to unchosen
siblings. We use the same context construction, variation and state-update
kernels, child capacities, and tie rule as in the genealogy-growth process.

\begin{proposition}[Family maximum and best-of-\(N\) greedy continuation]
\label{prop:app_maximum_greedy}
EvoTreeNAD with \(\Phi=\Phi_{\max}\) is pathwise equivalent to best-of-\(N\)
greedy continuation. Starting from the same initial state and using the same
random draws for variation and state updates, the two processes select the
same lineage and generate the same architecture and reward at every iteration.
\end{proposition}

\begin{proof}
Consider EvoTreeNAD with \(\Phi=\Phi_{\max}\). Before a node \(v\) becomes
saturated, routing stops at \(v\) whenever it reaches that node. Its children
therefore have no descendants when \(v\) first becomes saturated, so their
family maxima equal their immediate rewards. Let \(c^\star(v)\) be the
highest-reward child, with ties resolved by \(\prec\). An expansion below
\(c^\star(v)\) cannot decrease its family maximum and leaves the sibling
family values unchanged. The next visit to \(v\) therefore selects
\(c^\star(v)\) again. Induction over visits shows that this choice persists.
Thus EvoTreeNAD generates \(W(v)\) children before continuing through the
highest-reward child, without returning to unchosen siblings.

We now compare this process with best-of-\(N\) greedy continuation using the
same random draws. Their states agree initially. Suppose they agree before
an iteration. The selection property above gives the same root-to-node
lineage in both processes. Applying the same context construction to these
identical states and lineages gives the same variation context.
The kernel \(K\), evaluated at this context with
the same random draw, produces the same architecture and reward. The shared
state-update kernel \(\operatorname{Grow}\), using the same update draw,
then gives identical next states. Induction over iterations establishes the
claimed pathwise equivalence.
\end{proof}

\begin{proposition}[Support and bounded influence]
\label{prop:top_fraction_evidence}
Fix \(0<\tau\le1\). Let \(x=(x_1,\ldots,x_n)\in[a,b]^n\), with
\(x_{[1]}\ge\cdots\ge x_{[n]}\) denoting their descending order.
Set \(k=\max\{1,\lceil\tau n\rceil\}\).
The top-\(k\) average \(M_{k,n}(x)=k^{-1}\sum_{i=1}^k x_{[i]}\)
gives the top-percentile family value for these rewards and satisfies
\begin{equation}
M_{k,n}(x)
=
\max_{\substack{w_i\ge0,\ \sum_iw_i=1\\
\lVert w\rVert_\infty\le1/k}}
\sum_{i=1}^n w_i x_i.
\label{eq:top_fraction_capped_weights}
\end{equation}
For a fixed family size, changing one reward by \(\delta\) changes
\(M_{k,n}(x)\) by at most \(|\delta|/k\). Moreover, for
\(N_q=|\{i:x_i\ge q\}|\) and every \(q<b\),
\begin{equation}
\frac{N_q}{n}
\ge
\frac{k}{n}\frac{(M_{k,n}(x)-q)_+}{b-q}.
\label{eq:top_fraction_support_bound}
\end{equation}
Consequently, maintaining a family value a fixed amount above \(q\) as the
family grows requires a nonvanishing fraction of rewards at least \(q\).
Finitely many early high rewards cannot maintain this gap if subsequent
rewards are at most \(q\).
\end{proposition}

\begin{proof}
For \(1\le k\le n\), let
\[
\mathcal W_{n,k}
=
\left\{w\in\mathbb R_+^n:
\sum_iw_i=1,\ \lVert w\rVert_\infty\le\frac1k\right\}.
\]
If positive weight is assigned below the top \(k\) rewards while some top-
\(k\) coordinate has weight below \(1/k\), transferring the available mass to
the higher reward cannot decrease the objective. Repeating this exchange puts
weight \(1/k\) on the top \(k\) rewards and zero elsewhere, proving
Eq.~\eqref{eq:top_fraction_capped_weights}.

The variational form also gives
\[
\left|M_{k,n}(x+\delta e_j)-M_{k,n}(x)\right|
\le \sup_{w\in\mathcal W_{n,k}}w_j|\delta|
\le\frac{|\delta|}{k}.
\]
For the support bound, if \(N_q<k\), the top \(k\) rewards include all
\(N_q\) rewards at least \(q\). Bounding these by \(b\) and the remaining
\(k-N_q\) by \(q\) gives
\[
M_{k,n}(x)\le q+(b-q)\frac{N_q}{k}.
\]
If \(N_q\ge k\), this bound follows directly from
\(M_{k,n}(x)\le b\). Rearranging and using \(N_q\ge0\) gives
Eq.~\eqref{eq:top_fraction_support_bound}.

Finally, suppose only \(r<\infty\) rewards in a growing family exceed
\(q\), while all subsequent rewards are at most \(q\). When the family has rewards \(x_1,\ldots,x_N\),
\[
M_{k_\tau(N),N}(x_1,\ldots,x_N)
\le
q+\frac{\min\{r,k_\tau(N)\}}{k_\tau(N)}(b-q).
\]
Hence \(\limsup_{N\to\infty}M_{k_\tau(N),N}(x_1,\ldots,x_N)\le q\),
so the family value cannot maintain a fixed positive gap above \(q\).
The family value also eventually falls below that of any sibling with
a fixed family value greater than \(q\). This proves the proposition.
\end{proof}

\subsection{Stationary Selection-Variation}
\label{app:stationary_selection_variation}

We study the long-run behavior of the genealogy-growth process by jointly
tracking the reward law induced by each selected lineage context and the
realized reward. Theorem~\ref{thm:stationary_selection_variation}
establishes the existence of stationary empirical regimes for this joint
sequence. Within each regime, the distribution of generated rewards is
the mixture of the lineage-conditioned reward laws.
The theorem also bounds high-reward probability from below using the
average top-percentile mean across these conditional reward laws.

Let \(\mathcal F_t^\Phi\) denote the \(\sigma\)-algebra generated by the history
available before iteration \(t\). Define \(F_t^\Phi\) as the conditional law
of the next reward given this history, and let \(Z_t^\Phi\) denote the pair
consisting of this law and the realized reward,
\begin{equation}
F_t^\Phi
=\mathcal L(R_{t+1}^\Phi\mid\mathcal F_t^\Phi)
\in\mathcal P([a,b]),
\qquad
Z_t^\Phi=(F_t^\Phi,R_{t+1}^\Phi).
\label{eq:app_conditional_reward_law}
\end{equation}
\(F_t^\Phi\) is the reward marginal of \(K(\cdot\mid C_t^\Phi)\). Let
\(\mathsf E=\mathcal P([a,b])\times[a,b]\) be the state space of \(Z_t^\Phi\).
Equip \(\mathcal P([a,b])\) with the weak topology and use the corresponding
product topologies on \(\mathsf E\) and \(\mathsf E^{\mathbb N_0}\).
Define the empirical occupation measure on the sequence space
\(\mathsf E^{\mathbb N_0}\) by
\begin{equation}
\mathsf Q_{\Phi,T}
=\frac1T\sum_{t<T}\delta_{\sigma^t\mathbf Z^\Phi},
\qquad
\mathbf Z^\Phi=(Z_0^\Phi,Z_1^\Phi,\ldots),
\label{eq:app_empirical_orbit}
\end{equation}
where \(\sigma\) is the left shift on \(\mathsf E^{\mathbb N_0}\).
This is the empirical distribution of the shifted sequences
\(\sigma^t\mathbf Z^\Phi\), \(0\le t<T\).

A weak subsequential limit \(\mathsf Q\) of
\((\mathsf Q_{\Phi,T})_{T\ge1}\) is called a \emph{reachable stationary regime}.
Let \(\nu_{\mathsf Q}\) and \(\Lambda_{\mathsf Q}\) denote the two marginals of the
time-zero coordinate pair under \(\mathsf Q\), corresponding to the
conditional reward law and realized reward, respectively. These are the weak
limits of the respective empirical distributions along the same subsequence.
For \(0<\alpha\le1\), define the top-percentile mean and its average under
\(\nu_{\mathsf Q}\) by
\[
U_\alpha(F)=\frac1\alpha\int_{1-\alpha}^1F^{-1}(u)\,du,
\qquad
\overline U_{\alpha,\mathsf Q}
=\int U_\alpha(F)\,\nu_{\mathsf Q}(dF),
\]
where \(F^{-1}\) is the quantile function of \(F\), defined as the generalized
inverse of its cumulative distribution function.

\begin{theorem}[Stationary selection-variation]
\label{thm:stationary_selection_variation}
Every infinite realization of the genealogy-growth process admits weak
subsequential limits of \(\mathsf Q_{\Phi,T}\), and every such limit is
shift-stationary. Almost surely, every reachable stationary regime
\(\mathsf Q\) satisfies the following properties.
\begin{enumerate}
\item \emph{Reward distributions and exceedance frequencies.}
For every Borel set \(A\subseteq[a,b]\),
\begin{equation}
\Lambda_{\mathsf Q}(A)
=
\int_{\mathcal P([a,b])}F(A)\,\nu_{\mathsf Q}(dF).
\label{eq:stationary_selection_variation_mixture}
\end{equation}
If \(\mathsf Q_{\Phi,T_j}\Rightarrow\mathsf Q\), then at every threshold
\(q\) satisfying \(\Lambda_{\mathsf Q}(\{q\})=0\),
\begin{equation}
\lim_{j\to\infty}\frac1{T_j}\sum_{t<T_j}
\mathbf1\{R_{t+1}^\Phi\ge q\}
=
\Lambda_{\mathsf Q}([q,b])
=
\int F([q,b])\,\nu_{\mathsf Q}(dF).
\label{eq:stationary_realized_prevalence}
\end{equation}
For each fixed \(q<b\), the difference between the empirical exceedance
frequency and the average conditional exceedance probability converges to
zero almost surely along the full iteration sequence.

\item \emph{High-reward probability bound.}
For every \(0<\alpha\le1\) and \(q<b\),
\begin{equation}
\Lambda_{\mathsf Q}([q,b])
\ge
\frac{\alpha}{b-q}
\bigl(\overline U_{\alpha,\mathsf Q}-q\bigr)_+.
\label{eq:stationary_upper_fraction_bound}
\end{equation}
\end{enumerate}
\end{theorem}

\begin{proof}
Since \([a,b]\) is compact, \(\mathcal P([a,b])\) is compact metrizable
under the weak topology. Thus \(\mathsf E\) and its countable product
\(\mathsf E^{\mathbb N_0}\) are compact metrizable. By weak compactness of
\(\mathcal P(\mathsf E^{\mathbb N_0})\), the sequence
\((\mathsf Q_{\Phi,T})_{T\ge1}\) has a weakly convergent subsequence.
For every bounded continuous \(f\),
\[
\mathsf Q_{\Phi,T}(f-f\circ\sigma)
=\frac{f(\mathbf Z^\Phi)-f(\sigma^T\mathbf Z^\Phi)}{T}
\longrightarrow0.
\]
Every subsequential limit is shift invariant.

For \(h\in C([a,b])\), set
\[
D_{t+1}^{\Phi,h}
=h(R_{t+1}^\Phi)-\int h(r)F_t^\Phi(dr).
\]
These are bounded martingale differences. The martingale strong law, applied
simultaneously to a countable sup-norm-dense subset of \(C([a,b])\), gives
\(T^{-1}\sum_{t<T}D_{t+1}^{\Phi,h}\to0\) on one probability-one event.
The map \((F,r)\mapsto h(r)-\int h\,dF\), composed with the
first-coordinate projection, is bounded continuous on
\(\mathsf E^{\mathbb N_0}\).
Passing to any weak subsequential limit of \(\mathsf Q_{\Phi,T}\) and extending
by uniform approximation yields
\[
\int h(r)\Lambda_{\mathsf Q}(dr)
=
\int\!\left[\int h(r)F(dr)\right]\nu_{\mathsf Q}(dF)
\]
for every continuous \(h\). Equality of these probability measures proves
Eq.~\eqref{eq:stationary_selection_variation_mixture} for every Borel set.

The generated-reward marginal of
\(\mathsf Q_{\Phi,T_j}\) is
\(T_j^{-1}\sum_{t<T_j}\delta_{R_{t+1}^\Phi}\). Marginalization is continuous
under weak convergence, so these empirical reward measures converge weakly to
\(\Lambda_{\mathsf Q}\). If
\(\Lambda_{\mathsf Q}(\{q\})=0\), the set \([q,b]\) is a
\(\Lambda_{\mathsf Q}\)-continuity set. The Portmanteau theorem and
Eq.~\eqref{eq:stationary_selection_variation_mixture} then give
Eq.~\eqref{eq:stationary_realized_prevalence}.

For each fixed \(q<b\), the sequence
\(\mathbf1\{R_{t+1}^\Phi\ge q\}-F_t^\Phi([q,b])\)
is also a bounded martingale-difference sequence. Its strong law gives
\begin{equation}
\frac1T\sum_{t<T}
\left[
\mathbf1\{R_{t+1}^\Phi\ge q\}-F_t^\Phi([q,b])
\right]
\longrightarrow0
\quad\text{almost surely}.
\label{eq:stationary_prevalence_transfer}
\end{equation}

For any law \(F\), \(0<\alpha\le1\), and \(q<b\), the mass at or above
\(q\) within the highest-reward \(\alpha\) fraction is at most \(F([q,b])\).
Bounding rewards at or above \(q\) by \(b\) and those below \(q\) by \(q\) gives
\[
\alpha U_\alpha(F)\le\alpha q+(b-q)F([q,b]).
\]
Rearranging and using \(F([q,b])\ge0\) yields
\[
F([q,b])
\ge
\frac{\alpha}{b-q}\bigl(U_\alpha(F)-q\bigr)_+.
\]
Integrating under \(\nu_{\mathsf Q}\) and applying Jensen's inequality to the
positive-part map proves Eq.~\eqref{eq:stationary_upper_fraction_bound}.
\end{proof}

Theorem~\ref{thm:stationary_selection_variation} connects the variation
contexts explored by EvoTreeNAD to the architectures it produces during
continued genealogy growth. In each reachable stationary regime, if the
top-percentile means of the encountered conditional reward laws average
above \(q\), the theorem gives a positive lower bound on the probability of
generating architectures with rewards at least \(q\).

\subsection{Top-Percentile Routing and High-Reward Production}
\label{app:genealogy_flow_production}

Under the variation assumptions below, we connect sustained top-percentile
family values to high-reward production in these regimes and characterize
how mean and maximum routing differ.

\paragraph{Lineage-conditioned variation.}
We keep the component reward distributions fixed across lineage contexts to
compare how family-value rules respond to changes in high- and weak-outcome
probabilities.
Fix \(0<\tau<1\). Let \(r_0\) be the intermediate reward and \(q\) the
high-reward threshold, with \(a\le r_-<r_0<q\le r_+\le b\).
Let \(F_-\) and \(F_+\) be fixed weak- and high-reward distributions
supported on \([a,r_-]\) and \([r_+,b]\), respectively.
For a lineage context \(z\), define the conditional reward law
\begin{equation}
F_z
=
\ell(z)F_-
+\bigl(1-p_q(z)-\ell(z)\bigr)\delta_{r_0}
+p_q(z)F_+,
\label{eq:nad_opportunity_risk_law}
\end{equation}
where \(p_q(z)\) and \(\ell(z)\) denote the probabilities of high- and
weak-reward outcomes conditional on \(z\), respectively. We assume
\(p_q(z)+\ell(z)\le1\), \(0<p_q(z)<\tau\), and \(0<\ell(z)<1-\tau\).
The same conditional law is used by every routing rule at the same context.
Let \(\mu_-\) and \(\mu_+\) denote the means of \(F_-\) and \(F_+\),
respectively. The highest-reward \(\tau\) fraction contains all
high-reward outcomes and no weak outcomes. Therefore,
\begin{equation}
U_\tau(F_z)
=r_0+\frac{\mu_+-r_0}{\tau}p_q(z),
\qquad
\int rF_z(dr)
=r_0+(\mu_+-r_0)p_q(z)-(r_0-\mu_-)\ell(z).
\label{eq:app_opportunity_risk_scores}
\end{equation}
Averaging the reward laws of adaptively selected contexts changes only their
mixture weights; the component distributions remain fixed. Hence the same
identities hold for the averaged reward law.
Write \(F_{p,\ell}\) for the same mixture with high- and weak-outcome
probabilities \(p\) and \(\ell\). At iteration \(t\), the context is
\(C_t^\Phi\), so \(F_t^\Phi=F_{C_t^\Phi}\).

\paragraph{Recurrent families.}
Fix a family-value rule \(\Phi\) and an infinite realization of its genealogy-growth process.
For every created node \(v\), let \(N_T^\Phi(v)\) denote the number of
iterations before \(T\) whose selected lineage contains \(v\),
\begin{equation}
N_T^\Phi(v)=\sum_{t<T}\mathbf1\{v\in L_t^\Phi\}.
\label{eq:app_routed_call_count}
\end{equation}

Let \(\mathcal I_\Phi\) denote the set of recurrent nodes, which are visited
infinitely often by selected lineages,
\[
\mathcal I_\Phi=\{v:N_T^\Phi(v)\to\infty\}.
\]
For a saturated node \(v\), let \(\operatorname{Ch}(v)\) denote its children.
Its recurrent children form the set
\begin{equation}
\mathcal R_\Phi(v)
=
\{c\in\operatorname{Ch}(v):N_T^\Phi(c)\to\infty\}.
\label{eq:app_repeatedly_routed_children}
\end{equation}

For a recurrent node \(v\), let \(t_n^{\Phi,v}\) denote the \(n\)-th
iteration whose selected lineage contains \(v\),
\[
t_n^{\Phi,v}
=
\inf\{t\ge0:N_{t+1}^\Phi(v)=n\}.
\]
For every \(v\in\mathcal I_\Phi\), including the root, we assume
\begin{equation}
\frac1n\sum_{i=1}^n p_q(C_{t_i^{\Phi,v}}^\Phi)
\longrightarrow \rho_{\Phi,v}(q),
\qquad
\frac1n\sum_{i=1}^n \ell(C_{t_i^{\Phi,v}}^\Phi)
\longrightarrow \overline\ell_{\Phi,v}.
\label{eq:app_routed_average_resolution}
\end{equation}
The quantities \(\rho_{\Phi,v}(q)\) and \(\overline\ell_{\Phi,v}\) are the
limiting average conditional probabilities of high- and weak-reward outcomes,
respectively, over iterations whose selected lineages contain \(v\).
These limits may differ across policies because the policies generate different
contexts.

For every created non-root node \(c\) and every \(T\) at or after its creation,
let \(\widehat\Lambda_{c,T}\) denote the empirical reward law of its
descendant family. The family size and this law satisfy
\begin{equation}
|\mathcal H_T^\Phi(c)|=1+N_T^\Phi(c),
\qquad
\widehat\Lambda_{c,T}
=\frac{\delta_{R^\Phi(c)}+\sum_{t<T:c\in L_t^\Phi}\delta_{R_{t+1}^\Phi}}
{1+N_T^\Phi(c)}.
\label{eq:app_family_flow_identity}
\end{equation}
On the iteration that creates \(c\), routing stops at its parent. Every later iteration whose route
contains \(c\) inserts exactly one architecture into the subtree rooted at
\(c\), and every later descendant is created by one such iteration. This proves
Eq.~\eqref{eq:app_family_flow_identity}.

\begin{lemma}[Descendant-family law and family-value convergence]
\label{lem:app_routed_call_law}
For every non-root node \(c\), let \(I_t^c=\mathbf1\{c\in L_t^\Phi\}\).
On one probability-one event, simultaneously for every created non-root node
\(c\) and every \(h\in C([a,b])\),
\begin{equation}
N_T^\Phi(c)\to\infty
\quad\Longrightarrow\quad
\frac1{N_T^\Phi(c)}
\sum_{t<T}I_t^c
\left(h(R_{t+1}^\Phi)-\int h\,dF_t^\Phi\right)
\longrightarrow0.
\label{eq:app_selected_call_martingale}
\end{equation}
The same statement holds at the root with \(I_t^\rho\equiv1\) and
\(N_T^\Phi(\rho)=T\).

Under the stated conditions on the variation law in
Eq.~\eqref{eq:nad_opportunity_risk_law} and the convergence assumption in
Eq.~\eqref{eq:app_routed_average_resolution}, the empirical family law of every
recurrent non-root node \(c\) converges along the full sequence of iterations,
\begin{equation}
\widehat\Lambda_{c,T}
\Rightarrow
F_{\rho_{\Phi,c}(q),\overline\ell_{\Phi,c}}.
\label{eq:app_family_law_identification}
\end{equation}
The top-percentile and full-family mean values computed from the same
reward multiset \(\mathcal H_T^\Phi(c)\) satisfy
\begin{align}
\lim_{T\to\infty}\Phi_{\mathrm{top}}\!\left(\mathcal H_T^\Phi(c)\right)
&=
r_0+\frac{\mu_+-r_0}{\tau}\rho_{\Phi,c}(q),
\label{eq:app_top_limiting_score}\\
\lim_{T\to\infty}\Phi_{\mathrm{mean}}\!\left(\mathcal H_T^\Phi(c)\right)
&=
r_0+(\mu_+-r_0)\rho_{\Phi,c}(q)
-(r_0-\mu_-)\overline\ell_{\Phi,c}.
\label{eq:app_mean_limiting_score}
\end{align}
The empirical frequencies of \(R_{t+1}^\Phi\ge q\) and
\(R_{t+1}^\Phi\le r_-\) over the same routed iterations converge to
\(\rho_{\Phi,c}(q)\) and \(\overline\ell_{\Phi,c}\), respectively.
\end{lemma}

\begin{proof}
Fix \(h\in C([a,b])\) and a node address \(c\), taking \(I_t^c=0\) before
its creation, and write
\[
D_{t+1}=h(R_{t+1}^\Phi)-\int h\,dF_t^\Phi,
\qquad
\mathcal M_T=\sum_{t<T}\frac{I_t^cD_{t+1}}{1+N_t^\Phi(c)}.
\]
The weights are measurable with respect to the history before each iteration,
so \(\mathcal M_T\) is a martingale with
respect to \(\mathcal F_T^\Phi\). Since \(|D_{t+1}|\le2\|h\|_\infty\) and
the denominator at the \(n\)-th selected iteration is \(n\), orthogonality
of martingale increments gives
\[
\sup_T\mathbb E[\mathcal M_T^2]
\le4\|h\|_\infty^2\sum_{n\ge1}n^{-2}<\infty.
\]
Thus \(\mathcal M_T\) converges almost surely. On paths with \(N_T^\Phi(c)\to\infty\),
the series \(\sum_{n\ge1}D_{t_n^{\Phi,c}+1}/n\) converges, and
Kronecker's lemma gives \(n^{-1}\sum_{i=1}^nD_{t_i^{\Phi,c}+1}\to0\).
Between successive visits to \(c\), the numerator in
Eq.~\eqref{eq:app_selected_call_martingale} and the visit count
\(N_T^\Phi(c)\) remain unchanged. Applying this argument to a countable dense subset of the unit ball
of \(C([a,b])\), intersecting over the countable set of finite node addresses,
and extending by uniform approximation proves
Eq.~\eqref{eq:app_selected_call_martingale}. The same argument applies at the root.
For the two reward indicators, the conditional expectations are
\(p_q(C_t^\Phi)\) and \(\ell(C_t^\Phi)\).
The same argument, together with Eq.~\eqref{eq:app_routed_average_resolution},
yields both frequency limits.

The average conditional reward law over the first \(n\) iterations routed through \(c\) is
\[
\frac1n\sum_{i=1}^nF_{C_{t_i^{\Phi,c}}^\Phi}
=
F_{\overline p_{\Phi,c,n},\overline\ell_{\Phi,c,n}},
\]
where \(\overline p_{\Phi,c,n}\) and \(\overline\ell_{\Phi,c,n}\) denote
the average conditional probabilities of high- and weak-reward outcomes
over these \(n\) visits, respectively.
By Eq.~\eqref{eq:app_routed_average_resolution}, this mixture converges to
\(F_{\rho_{\Phi,c}(q),\overline\ell_{\Phi,c}}\).
The empirical family law in Eq.~\eqref{eq:app_family_flow_identity} has the
same weak limit. For each continuous test function,
Eq.~\eqref{eq:app_selected_call_martingale} shows that the martingale average
converges to zero. The reward of \(c\) has vanishing weight in the empirical
family law and therefore does not affect its weak limit.
This proves Eq.~\eqref{eq:app_family_law_identification} along
the full sequence of iterations.

Let \(m_T=|\mathcal H_T^\Phi(c)|\) denote the family size and
\(\alpha_T=k_\tau(m_T)/m_T\) the fraction of family rewards included in
the top-percentile average. Then \(\alpha_T\ge\tau\) and
\(\alpha_T\to\tau\), and the empirical top-percentile family value is
\(U_{\alpha_T}(\widehat\Lambda_{c,T})\).
Let \(F_*=F_{\rho_{\Phi,c}(q),\overline\ell_{\Phi,c}}\).
On the bounded reward interval, weak convergence gives
\(W_1(\widehat\Lambda_{c,T},F_*)\to0\), where \(W_1\) denotes the
Wasserstein--1 distance. Its quantile representation and \(\alpha_T\ge\tau\) give
\[
\left|U_{\alpha_T}(\widehat\Lambda_{c,T})-U_\tau(F_*)\right|
\le
\frac{W_1(\widehat\Lambda_{c,T},F_*)+(b-a)(\alpha_T-\tau)}{\alpha_T}
\longrightarrow0.
\]
The identity for \(U_\tau\) in Eq.~\eqref{eq:app_opportunity_risk_scores}
also holds at the limiting mixture weights, giving
Eq.~\eqref{eq:app_top_limiting_score}. Weak convergence on the bounded reward
interval also gives convergence of the means, proving
Eq.~\eqref{eq:app_mean_limiting_score}.
\end{proof}

Lemma~\ref{lem:app_routed_call_law} establishes family-value convergence
for adaptively accumulated descendant rewards. We now use the routing rule
to relate these limits across recurrent families and connect top-percentile
family values to the high-reward production rate of the run.
Let \(A=\mu_+-r_0>0\) and \(B=r_0-\mu_->0\).

\begin{theorem}[Top-percentile routing and recurrent high-reward production]
\label{thm:genealogy_flow_production}
Under the stated conditions on the variation law in
Eq.~\eqref{eq:nad_opportunity_risk_law} and the convergence assumption in
Eq.~\eqref{eq:app_routed_average_resolution}, the following conclusions hold
almost surely.
\begin{enumerate}
\item \emph{Top-percentile.} For every fixed non-root
\(v\in\mathcal I_{\mathrm{top}}\) and every reachable stationary regime
\(\mathsf Q\) of the same top-percentile run,
\begin{equation}
\lim_{t\to\infty}V_t^{\mathrm{top}}(v)
=
U_\tau(\Lambda_{\mathsf Q})
=
r_0+\frac{\mu_+-r_0}{\tau}\Lambda_{\mathsf Q}([q,b]).
\label{eq:app_top_stationary_value}
\end{equation}
For every \(v\in\mathcal I_{\mathrm{top}}\), all recurrent children
\(c\in\mathcal R_{\mathrm{top}}(v)\) satisfy
\(\rho_{\mathrm{top},c}(q)=\rho_{\mathrm{top},v}(q)\).
\item \emph{Full-family mean.} For every \(v\in\mathcal I_{\mathrm{mean}}\),
all children \(c\in\mathcal R_{\mathrm{mean}}(v)\) attain the same maximal
limiting family value among all children of \(v\).
The quantity \(A\rho_{\mathrm{mean},c}(q)-B\overline\ell_{\mathrm{mean},c}\)
is constant over \(c\in\mathcal R_{\mathrm{mean}}(v)\).

The gap between the maximal high-reward frequency among recurrent children
and that over iterations routed through \(v\) satisfies
\begin{equation}
\begin{aligned}
&\max_{c\in\mathcal R_{\mathrm{mean}}(v)}\rho_{\mathrm{mean},c}(q)
-\rho_{\mathrm{mean},v}(q)\\
&\quad=\frac BA
\left(
\max_{c\in\mathcal R_{\mathrm{mean}}(v)}\overline\ell_{\mathrm{mean},c}
-\overline\ell_{\mathrm{mean},v}
\right).
\end{aligned}
\label{eq:app_mean_production_gap}
\end{equation}
This gap is positive if and only if there exists
\(c\in\mathcal R_{\mathrm{mean}}(v)\) such that
\[
\overline\ell_{\mathrm{mean},c}
<
\max_{d\in\mathcal R_{\mathrm{mean}}(v)}\overline\ell_{\mathrm{mean},d},
\qquad
\limsup_{T\to\infty}
\frac{N_T^{\mathrm{mean}}(c)}{N_T^{\mathrm{mean}}(v)}>0.
\]
\item \emph{Family maximum.} EvoTreeNAD with maximum-family routing is
pathwise equivalent to matched best-of-\(N\) greedy continuation
(Proposition~\ref{prop:app_maximum_greedy}).
\end{enumerate}
The following identities hold for all recurrent nodes in the respective
genealogies,
\begin{equation}
\begin{aligned}
\rho_{\mathrm{top},v}(q)&=\rho_{\mathrm{top},\rho}(q),
&&v\in\mathcal I_{\mathrm{top}},\\
A\rho_{\mathrm{mean},v}(q)-B\overline\ell_{\mathrm{mean},v}
&=A\rho_{\mathrm{mean},\rho}(q)-B\overline\ell_{\mathrm{mean},\rho},
&&v\in\mathcal I_{\mathrm{mean}}.
\end{aligned}
\label{eq:app_recurrent_genealogy_production}
\end{equation}
For each rule, the realized long-run frequency of architectures with reward at
least \(q\) equals \(\rho_{\Phi,\rho}(q)\).
\end{theorem}

\begin{proof}
Let \(\Phi\in\{\Phi_{\mathrm{top}},\Phi_{\mathrm{mean}}\}\) and fix
\(v\in\mathcal I_\Phi\). Since \(v\) is recurrent and has finite child
capacity, it eventually becomes saturated.
Lemma~\ref{lem:app_routed_call_law} establishes convergence of the
family value along the full iteration sequence for every recurrent child.
Every other child is routed through only finitely many times, so its family
value is eventually constant. Thus all
children of \(v\) have limiting family values.
Each route through a saturated node enters one of finitely many children,
so every recurrent node has at least one recurrent child. Every ancestor
of a recurrent node is also recurrent.

Suppose a child \(c\in\mathcal R_\Phi(v)\) had a limiting family value strictly below that of
another child \(d\). For all sufficiently large \(t\),
\(V_t^\Phi(d)>V_t^\Phi(c)\). Whenever a later route reaches \(v\), rooted
routing therefore cannot enter \(c\). Its visitation count would become
constant, contradicting \(N_T^\Phi(c)\to\infty\). Hence every recurrent
child attains the same maximal limiting family value among all children
of \(v\). Under top-percentile routing,
Eq.~\eqref{eq:app_top_limiting_score} implies that the recurrent child
families have the same limiting high-reward frequency. Under mean routing,
Eq.~\eqref{eq:app_mean_limiting_score} shows that
\(A\rho_{\mathrm{mean},c}(q)-B\overline\ell_{\mathrm{mean},c}\)
is constant over \(c\in\mathcal R_{\mathrm{mean}}(v)\).

To relate the limiting high- and weak-reward frequencies for \(v\) and its
children, define the local visitation ratios
\(w_T(c)=N_T^\Phi(c)/N_T^\Phi(v)\). After the finite pre-saturation phase,
each route through \(v\) passes through exactly one child. Thus
\[
\sum_{t<T:v\in L_t^\Phi}p_q(C_t^\Phi)
=
O(1)+
\sum_{c\in\operatorname{Ch}(v)}
\sum_{t<T:c\in L_t^\Phi}p_q(C_t^\Phi),
\]
with the same identity for \(\ell\).
Children outside \(\mathcal R_\Phi(v)\) contribute only finitely many terms.
Dividing by \(N_T^\Phi(v)\) and applying
Eq.~\eqref{eq:app_routed_average_resolution} gives, as \(T\to\infty\),
\begin{equation}
\begin{aligned}
\rho_{\Phi,v}(q)&=
\sum_{c\in\mathcal R_\Phi(v)}w_T(c)\rho_{\Phi,c}(q)+o(1),\\
\overline\ell_{\Phi,v}&=
\sum_{c\in\mathcal R_\Phi(v)}w_T(c)\overline\ell_{\Phi,c}+o(1),
\qquad
\sum_{c\in\mathcal R_\Phi(v)}w_T(c)\longrightarrow1.
\end{aligned}
\label{eq:app_flow_weighted_production}
\end{equation}
These identities remain valid when \(N_T^\Phi(v)/T\) tends to zero or the
local visitation ratios \(w_T(c)\) oscillate.

For top-percentile routing, the common high-reward frequency among recurrent
child families and Eq.~\eqref{eq:app_flow_weighted_production} give
\begin{equation}
\rho_{\mathrm{top},v}(q)
=
\max_{c\in\mathcal R_{\mathrm{top}}(v)}
\rho_{\mathrm{top},c}(q).
\label{eq:top_flow_production_identity}
\end{equation}
Thus \(\rho_{\mathrm{top},c}(q)=\rho_{\mathrm{top},v}(q)\)
for every \(c\in\mathcal R_{\mathrm{top}}(v)\).
Applying this equality along the finite path from the root gives the
top-percentile relation in Eq.~\eqref{eq:app_recurrent_genealogy_production}.

For each of the three rules, \(N_T^\Phi(\rho)=T\), and the indicator case of
Lemma~\ref{lem:app_routed_call_law} identifies the root rate with the realized
high-reward frequency.

To identify the stationary reward law, apply
Lemma~\ref{lem:app_routed_call_law} at the root to continuous test functions.
Together with Eq.~\eqref{eq:app_routed_average_resolution}, this gives
\[
\frac1T\sum_{t<T}\delta_{R_{t+1}^\Phi}
\Rightarrow
F_{\rho_{\Phi,\rho}(q),\overline\ell_{\Phi,\rho}}
\quad\text{along the full iteration sequence}.
\]
The left side is the generated-reward marginal of
\(\mathsf Q_{\Phi,T}\). Hence every reachable stationary regime
\(\mathsf Q\) of this run has
\(\Lambda_{\mathsf Q}
=F_{\rho_{\Phi,\rho}(q),\overline\ell_{\Phi,\rho}}\).
Equation~\eqref{eq:app_opportunity_risk_scores} also holds at the limiting
mixture weights, so
\[
\Lambda_{\mathsf Q}([q,b])=\rho_{\Phi,\rho}(q),
\qquad
U_\tau(\Lambda_{\mathsf Q})
=r_0+\frac{\mu_+-r_0}{\tau}\rho_{\Phi,\rho}(q).
\]
For \(\Phi=\Phi_{\mathrm{top}}\), combine this identity with
Eqs.~\eqref{eq:app_top_limiting_score}
and~\eqref{eq:app_recurrent_genealogy_production} to obtain
Eq.~\eqref{eq:app_top_stationary_value}. The common probability-one event in
Lemma~\ref{lem:app_routed_call_law} makes the conclusion simultaneous for all
fixed recurrent families and all reachable regimes of the run.

For mean routing, the common limiting family value of recurrent children and
Eq.~\eqref{eq:app_flow_weighted_production} give
\[
A\rho_{\mathrm{mean},c}(q)-B\overline\ell_{\mathrm{mean},c}
=A\rho_{\mathrm{mean},v}(q)-B\overline\ell_{\mathrm{mean},v}
\quad(c\in\mathcal R_{\mathrm{mean}}(v)).
\]
Subtracting and taking maxima over the same child set proves
Eq.~\eqref{eq:app_mean_production_gap}.

Let \(\ell_*=\max_{c\in\mathcal R_{\mathrm{mean}}(v)}
\overline\ell_{\mathrm{mean},c}\). Equation~\eqref{eq:app_flow_weighted_production}
implies
\[
\ell_*-\overline\ell_{\mathrm{mean},v}
=
\lim_{T\to\infty}
\sum_{c\in\mathcal R_{\mathrm{mean}}(v)}
w_T(c)\bigl(\ell_*-\overline\ell_{\mathrm{mean},c}\bigr).
\]
There are finitely many nonnegative summands. The limit is positive exactly
when a child with \(\overline\ell_{\mathrm{mean},c}<\ell_*\) has
\(\limsup_T w_T(c)>0\), proving the strictness criterion.

Applying the parent-child identity under mean routing along the finite
path from the root gives the mean relation in
Eq.~\eqref{eq:app_recurrent_genealogy_production}. This holds simultaneously
for all recurrent nodes on the probability-one event of
Lemma~\ref{lem:app_routed_call_law}.

The claim for maximum-family routing follows from
Proposition~\ref{prop:app_maximum_greedy}.
\end{proof}

Under the stated variation assumptions,
Theorem~\ref{thm:genealogy_flow_production} shows that the limiting
top-percentile value of each recurrent non-root family exactly characterizes
the probability of generating high-reward architectures in the stationary
regimes of the run. Sustaining a high value in a growing family requires a
nonvanishing fraction of high-reward descendants, not just isolated early
successes (Proposition~\ref{prop:top_fraction_evidence}). Full-family means
also depend on weak outcomes and can assign equal limiting values to families
with different high-reward production rates. Maximum-family routing fixes
the winning child at each generation and is pathwise equivalent to
best-of-\(N\) greedy continuation
(Proposition~\ref{prop:app_maximum_greedy}). EvoTreeNAD can therefore continue
developing successful architectures and redirect evolution toward retained
alternatives as new descendant evidence accumulates.

\FloatBarrier


\section{Additional Method Details}

\subsection{Hyperparameters}
\label{appendix:evotree_hparams}

Table~\ref{tab:evotree_hparams} summarizes the hyperparameter
definitions and representative settings used across tasks. Values marked
task-dependent vary with the dataset and discovery budget. Representative
reward anchors include $(v_e,v_l)=(0.8,0.92)$ for CIFAR-10 and
$(0.55,0.72)$ for CIFAR-100.

\begin{table}[htbp]
\centering
\caption{EvoTreeNAD hyperparameters, definitions, and representative settings. Task-dependent values vary by dataset and budget.}
\label{tab:evotree_hparams}
\resizebox{\textwidth}{!}{%
\begin{tabular}{lp{1.5cm}p{0.55\linewidth}l}
\toprule
\textbf{Name} & \textbf{Symbol} & \textbf{Description} & \textbf{Representative setting} \\
\midrule
Child capacity & $W(s)$ & Maximum number of direct children of node $s$. Capacities may differ across nodes; for example, $W(\rho)=15$ at the root and $W(s)=7$ at non-root nodes. & task-dep. \\
Code realization attempts & $A_{\max}$ & Maximum Code-Agent attempts within one proposed variation. & 1--3 \\
Rejection tolerance & $P_{\text{tolerant}}$ & Failed expansions skipped per parent before unresolved failures are recorded. & 4--8 \\
Top-percentile fraction & $\tau$ & Fraction of the highest rewards averaged over a node and all its descendants (Eq.~\ref{eq:value_family_top_tau}). & 0.1--0.2 \\
Idea ancestors & $k_{\text{idea}}$ & Maximum number of recent ancestors provided to the Idea Agent as context. & 1--3 \\
Code ancestors & $k_{\text{code}}$ & Maximum number of recent ancestors provided to the Code Agent as context. & 1--3 \\
Early-stop reference fraction & $q_{\text{stop}}$ & Fraction of prior candidates selected for strong performance to construct reference validation trajectories. & 0.1--0.2 \\
Early-stop margin & $\delta$ & Base margin below reference performance, reduced at later training stages. & 0.01--0.08 \\
Reward center & $v_e$ & Metric value corresponding to the midpoint of the sigmoid reward mapping (Eq.~\ref{eq:reward_sigmoid_clean}). & task-dep. \\
Reward target & $v_l$ & Target metric value to approach in reward mapping. & task-dep. \\
Reward slope & $\kappa$ & Sharpness of sigmoid transition. & 2 \\
Size coefficient & $\lambda_{\mathrm{size}}$ & Coefficient for the compactness term in the reward (Eq.~\ref{eq:reward_sigmoid_clean}). & 0.0--0.01 \\
\midrule
Short-budget training & $\{\eta, T_{\max},$ $B, T_{\text{eval}},$ $ T_{\text{runtime}}\}$ & Short-budget training settings: learning rate $\eta$, max steps $T_{\max}$, batch size $B$, evaluation interval $T_{\text{eval}}$, and training runtime constraint $T_{\text{runtime}}$. & task-dep. \\
\bottomrule
\end{tabular}
}
\end{table}

For variation from an architecture node, agent context includes the selected
architecture and its evaluation results, together with records from up to
$k_{\text{idea}}$ recent ancestors for the Idea Agent and $k_{\text{code}}$
for the Code Agent. These limits retain recent design history while controlling
context length.

EvoTreeNAD optionally adds small random perturbations to sibling node-family
values before branch selection to encourage exploration of alternative
lineages.

\subsection{Discovery-Stage Evaluation}
\label{app:discovery_evaluation}

LLM-generated architectures are not guaranteed to execute correctly or train
successfully, and even trainable candidates may perform poorly. Our
discovery-stage evaluation is designed to accommodate these generation failures
and assess candidate performance efficiently, providing the validation feedback
needed for continued evolution.

\paragraph{Execution checks.}
Generated implementations undergo a smoke test for executability and
compatibility with the task interface. Only implementations that pass proceed
to short-budget evaluation.

\FloatBarrier
\paragraph{Short-budget evaluation with early stopping.}
During discovery, architectures are trained under a fixed short-budget recipe
and evaluated on the task validation set. We use early stopping to reduce
computation spent on poorly performing candidates. The reference group comprises
the strongest prior candidates among those with the longest available validation
histories, ranked by their final recorded scores. The reference fraction
$q_{\text{stop}}$ controls the size of this group. For the accuracy and AUC
metrics used here, the stopping threshold is based on the lowest score in this
group at the same training step, with $\delta$ setting the base allowable
shortfall below that score. Training stops when the candidate's validation score
falls below this threshold. Recorded validation results, including those obtained
before early stopping, provide the discovery scores used to compute node rewards.

\paragraph{Failure handling.}
Failed execution checks or training attempts that yield no valid score prompt
the Code Agent to revise its implementation using diagnostic feedback. Each
proposed variation allows at most $A_{\max}$ Code-Agent attempts, including the
initial attempt, to limit the cost of repeated attempts. If none produces a
valid score, the expansion is counted as failed. For each parent, the first
$P_{\text{tolerant}}$ failed expansions are skipped without adding a child.
This rejection tolerance accommodates occasional unsuccessful variations.
Further failed expansions are recorded under that parent with a fixed
negative reward, so persistent failures contribute to family values.
Intermediate retry attempts are not inserted as separate children.

\FloatBarrier

\section{Model Interface}
\label{appendix:model_contract}

The CIFAR-10 model interface contract used in our experiments is shown below.
Other tasks use analogous contracts for initialization, model inputs and
outputs, and task-specific requirements. Training uses the loss returned by
\texttt{forward}, while evaluation uses the logits returned by
\texttt{predict} and the fixed task metric.

\begin{CodeBlock}
"""
model_requirement defines the structural and functional conditions that a model must satisfy for a particular predictive task. These conditions specify interface compliance; predictive quality is determined by evaluator feedback.
For functions in the model:
 - Input signature: Each method takes one or more named tensor arguments, specified by parameter names and their shapes/types.
 - Output format: Each method returns a dictionary where each key corresponds to a named output tensor with a defined shape and type.
 - Ensure the methods include the specified arguments or keys, but they don't need to be limited to them. Adding extra and useful keys to the output dictionary is encouraged to support interaction between modules.
     - For example, you are free to output additional keys in the outputs dictionary.
     
Note: We use `predict` to output the predicted logit for evaluation. Do not use the ground truth label as the input of the `predict` function. We use `forward` to train the model. Use CPU for executable validation.
Contract: `forward` and `predict` method must be tolerant of extra, unknown keyword arguments - implement the signature with **kwargs and silently ignore unrecognized keys.
"""

model_requirements = {
    "model_name": "ImageClfModel",
    "purpose": "Predict label of image",
    "data_background": "label_num is 10. the pixel_values are normalized.",
    "init_parameters": {"label_num": {"type": "int"}, "base_dim": {"type": "int"}, "model_depth": {"type": "int"}},
    "methods": {
        "forward": {
            "inputs": {
                "pixel_values": {"shape": "(batch_size, 3, 32, 32)", "type": "torch.FloatTensor"},
                "label": {"shape": "(batch_size)", "type": "torch.LongTensor"},
            },
            "outputs": {
                "logits": {"shape": "(batch_size, label_num)", "type": "torch.FloatTensor"},
                "loss": {"shape": "()", "type": "torch.FloatTensor"}
            }
        },
        "predict": {
            "inputs": {
                "pixel_values": {"shape": "(batch_size, 3, 32, 32)", "type": "torch.FloatTensor"},
            },
            "outputs": {
                "logits": {"shape": "(batch_size, label_num)", "type": "torch.FloatTensor"},
            }
        },
    },
    "other_requirements": ["The `base_dim` is typically set to 32 and defines a base dim of the model. Specific modules adjust this dimension by expanding or reducing it as needed. (base_dim*2 = 64, base_dim*4=128, base_dim*8=256, ...)",
                          "The `model_depth` defines the number of main layers or blocks in the model at a macro level - typically values like 12, 15, or 18. To create layers or stages, adapt them based on `model_depth`. E.g., use a safe splitting function or define stage_depth = ratio * model_depth.",
                          "Sample size is about 50k, not large. Be cautious about overfitting and early saturation."]
}
\end{CodeBlock}

\FloatBarrier

\section{Dataset and Protocol Details}

\subsection{Discovery Protocol}
\label{appendix:discovery_protocol}

Main discovery uses three independent EvoTreeNAD runs per task. We analyze
continued architecture evolution in two 300-iteration CIFAR-10 discovery runs.
MedMNIST-v2 runs use 80 iterations. CIFAR-10 discovery-strategy and
agent-configuration comparisons use 100 iterations per run.
Prompts use generic dataset paths during discovery.

Main results use GPT-4.1 for architectural proposals and OSS20B for code
realization. GPT-5 is used only in the supplementary agent-configuration study.
GPT-4.1 and GPT-5 are accessed through Azure OpenAI
(\texttt{gpt-4.1} and \texttt{gpt-5-chat}); OSS20B is run locally from the
\texttt{gpt-oss-20b-mxfp4.gguf} checkpoint on two V100 GPUs. Each discovery
run uses one A100 GPU for candidate evaluation, and full-fidelity training uses
A100 or V100 GPUs. Representative hyperparameter ranges and prompt templates appear in
Appendices~\ref{appendix:evotree_hparams} and~\ref{appendix:prompts}.

\subsection{MedMNIST-v2 Tasks and Data}
\label{appendix:medmnist_description_statistics}

We evaluate EvoTreeNAD on six MedMNIST-v2 medical image classification tasks~\citep{medmnistv2} spanning 2D images and 3D volumes across multiple imaging modalities. Table~\ref{tab:medmnist_stats} summarizes the metadata and split statistics for each dataset. The tasks are:
\begin{itemize}
    \item \textbf{PathMNIST:} Multi-class classification of colon pathology patches extracted from hematoxylin and eosin (H\&E) stained histological images.
    \item \textbf{OCTMNIST:} Multi-class diagnosis of retinal diseases from optical coherence tomography (OCT) scans.
    \item \textbf{TissueMNIST:} Eight-class classification of human kidney cortex cells from microscopy images.
    \item \textbf{VesselMNIST3D:} Binary classification of healthy and aneurysmal intracranial vessel segments using MRA-derived 3D shapes.
    \item \textbf{SynapseMNIST3D:} Binary classification of excitatory and inhibitory synapses from electron microscopy.
    \item \textbf{OrganMNIST3D:} Eleven-class classification of body organs from CT scans.
\end{itemize}

\begin{table}[htbp]
    \centering
    \caption{MedMNIST-v2 Dataset Statistics.}
    \label{tab:medmnist_stats}
    \resizebox{\textwidth}{!}{%
    \begin{tabular}{l l l l c c c c}
        \toprule
        \textbf{Dataset} &
        \textbf{Data Modality} &
        \textbf{Task Type} &
        \textbf{Input Shape} &
        \textbf{\#(Train)} &
        \textbf{\#(Val)} &
        \textbf{\#(Test)} &
        \textbf{Total} \\
        \midrule
        \textbf{PathMNIST} & Colon Pathology (Histology) & Multi-Class (9) & $3 \times 64 \times 64$ & 89,996 & 10,004 & 7,180 & 107,180 \\
        \textbf{OCTMNIST} & Retinal OCT & Multi-Class (4) & $1 \times 64 \times 64$ & 97,477 & 10,832 & 1,000 & 109,309 \\
        \textbf{TissueMNIST} & Kidney Cortex Microscopy & Multi-Class (8) & $1 \times 64 \times 64$ & 165,466 & 23,640 & 47,280 & 236,386 \\
        \midrule
        \textbf{VesselMNIST3D} & Brain MRA (Shape) & Binary (2) & $1 \times 64 \times 64 \times 64$ & 1,335 & 191 & 382 & 1,908 \\
        \textbf{SynapseMNIST3D} & Electron Microscopy & Binary (2) & $1 \times 64 \times 64 \times 64$ & 1,230 & 177 & 352 & 1,759 \\
        \textbf{OrganMNIST3D} & Abdominal CT & Multi-Class (11) & $1 \times 64 \times 64 \times 64$ & 971 & 161 & 610 & 1,742 \\
        \bottomrule
    \end{tabular}
    }
\end{table}

\subsection{Full-Fidelity Evaluation Protocol}
\label{app:full_training_details}

All architectures selected after evolution are trained from scratch on the full
official training split of each dataset under the task-specific full-fidelity
settings summarized in Table~\ref{tab:full_train_protocol}. The additional regularization and
augmentation used at this stage are fixed within each task and are not tuned per
architecture.

\begin{table}[htbp]
\centering
\caption{Full-fidelity training settings for discovered models. All runs use the AdamW optimizer and a cosine-decay learning-rate schedule.}
\label{tab:full_train_protocol}
\begin{tabular}{lccc}
\toprule
\textbf{Dataset} & \textbf{Epochs} & \textbf{Batch Size} & \textbf{Learning Rate} \\
\midrule
CIFAR-10         & 500   & 128 & 0.002 \\
CIFAR-100        & 500   & 128 & 0.001 \\
MedMNIST-v2 2D      & 60    & 160 & 0.001 \\
MedMNIST-v2 3D      & 200   & 64  & 0.0005 \\
\addlinespace[2pt]
\multicolumn{4}{l}{\footnotesize CIFAR-10/100: DropPath, MixUp, and CutMix.} \\
\multicolumn{4}{l}{\footnotesize MedMNIST-v2 2D: DropPath and MixUp; 3D: DropPath.} \\
\bottomrule
\end{tabular}
\end{table}

\FloatBarrier

\section{Additional Experimental Analyses}
\label{sec:detailed_ana}

These analyses examine the architectures discovered by EvoTreeNAD and their
development through continued evolution. We present
representative discovered architectures, architecture performance along principal
lineages, continued architecture evolution, illustrations of discovery strategies,
and agent configurations and costs.

\subsection{Representative Discovered Architectures}
\label{appendix:representative_architectures}

Table~\ref{tab:discovered_architecture_structures} summarizes the structures of
representative architectures discovered across the evaluated task families.

\begin{table}[!htbp]
\centering
\caption{Representative architectures discovered by EvoTreeNAD from an empty root.}
\label{tab:discovered_architecture_structures}
\footnotesize
\setlength{\tabcolsep}{4pt}
\begin{tabular}{@{}p{0.17\linewidth}p{0.79\linewidth}@{}}
\toprule
Task family & Representative discovered structures \\
\midrule
CIFAR-10
& Dynamic routing with stage transformers and cross-stage fusion; fractal
MBConv with pyramid pooling. \\
CIFAR-100
& Mixed-kernel MBConv with FFT context, ODE blocks, and cross-stage
concatenation; Res2Net-style branches with gated expansion and channel
attention. \\
MedMNIST-v2 2D
& Selective-kernel layers with GeM pooling; inverted-residual hierarchies with
squeeze-and-excitation. \\
MedMNIST-v2 3D
& Depthwise-separable or MBConv-style 3D residual networks with task-specific
frequency blocks, normalization, and anti-aliased downsampling. \\
\bottomrule
\end{tabular}
\end{table}

\subsection{Architecture Performance Along Principal Lineages}
\label{appendix:evo_progress_ana}

To examine the quality of discovered architectures along the principal lineage
extracted after each run, we evaluate three architectures spanning its early and
later stages under the full-fidelity protocol.

\begin{itemize}
    \item \textbf{Architecture I} (Selected Node 1): the first architecture after the empty root, providing a direct-generation reference.
    \item \textbf{Architecture II} (Selected Node 2): the highest-reward architecture in the first half of the lineage after excluding Architecture I.
    \item \textbf{Architecture III} (Selected Node 3): the highest-reward architecture in the second half of the lineage.
\end{itemize}

\begin{figure}[!htbp]
    \centering
    \includegraphics[width=0.78\linewidth]{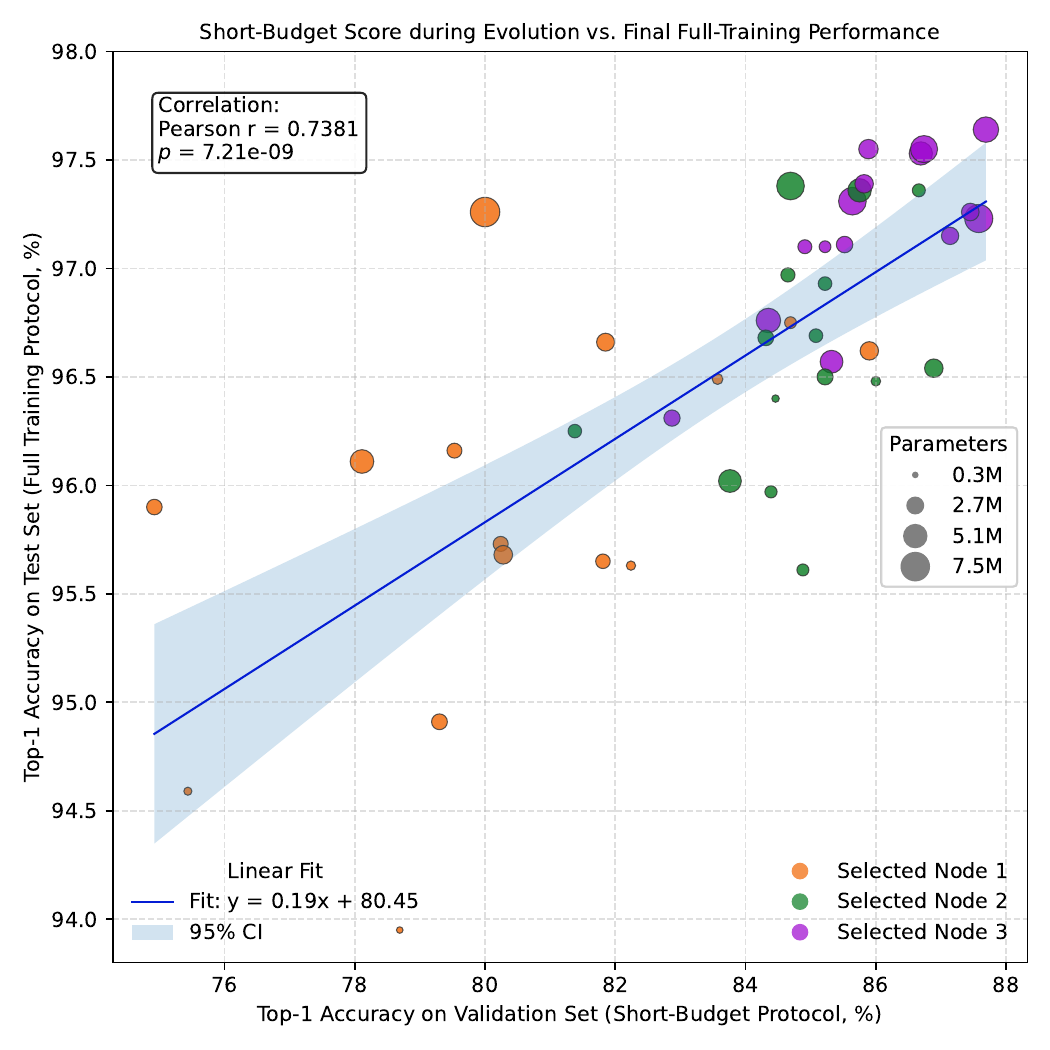}
    \caption{Correlation between short-budget and full-fidelity performance on
    selected CIFAR-10 principal-lineage architectures.}
    \label{fig:short_vs_full}
\end{figure}

Figure~\ref{fig:main_lineage_progression} presents full-fidelity performance
across three architectures on each principal lineage. Performance rises and
falls across the evaluated architectures, with the best results frequently
occurring later along the lineage. Parameter counts often decrease alongside modest
accuracy gains, consistent with the auxiliary size term in
Eq.~\ref{eq:reward_sigmoid_clean}.

Architectures generated directly from the empty root (Architecture I) can already
achieve competitive performance, showing that empty-root generation can provide
strong starting points. Relative to these empty-root references, later architectures more consistently
achieve high performance.

Figure~\ref{fig:short_vs_full} shows the correlation between short-budget and
full-fidelity performance of these CIFAR-10 architectures. Because
they lie on principal lineages extracted after discovery, the correlation assesses
whether discovery-stage validation scores remain informative for their
full-fidelity ranking. The positive relationship indicates that, at the CIFAR-10
benchmark scale studied here, short-budget evaluation provides a useful
screening signal during evolution, consistent with prior work on efficient
architecture evaluators
\citep{abdelfattah2021zerocostproxieslightweightnas,ru2021speedyperformanceestimation,white2021powerfulperformance,ning2021evaluatingefficientperformance}.

\subsection{Continued Architecture Evolution}
\label{sec:extend_iter}

We analyze architecture development over 300 iterations in two EvoTreeNAD
runs on CIFAR-10, one without an Idea Agent and the other with GPT-4.1 as
the Idea Agent. Both use OSS20B as the Code Agent. These promising runs
were selected from the 100-iteration experiments reported in
Appendix~\ref{appendix:evo_progress_ana} and continued to 300 iterations.

For this analysis, we extracted each principal lineage and selected
for full-fidelity evaluation the architectures whose discovery rewards improved
over those of their immediate predecessors.

Figures~\ref{fig:main_extended_proxy} and~\ref{fig:main_extended_full}
present the discovery-stage validation accuracy of all architectures
along each principal lineage and the full-fidelity performance of the
evaluated architectures. Later evolution in the run using the GPT-4.1
Idea Agent discovered our best-performing CIFAR-10 architecture.
The run without an Idea Agent followed a more gradual trajectory
while also improving its full-fidelity accuracy.
Discovery-stage and full-fidelity trajectories remain non-monotone.
Together, these runs show that continued evolution in EvoTreeNAD can
yield further gains in both short-budget validation accuracy and
full-fidelity performance.
The configuration study in Appendix~\ref{appendix:agent_configs}
separately compares the proposal and realization roles across
independent runs.

\subsection{Illustration of Discovery Strategies}
\label{sec:ablation_tree}

Figure~\ref{fig:ablation_schematic} illustrates the genealogies under the four
discovery strategies compared in Table~\ref{tab:discovery_strategy_controls}.
EvoTreeNAD retains alternative branches and uses top-percentile family
values to select lineages for continued evolution.
Best-of-\(N\) greedy continuation, labeled Greedy Single-Lineage in
panel~(b), generates sibling architectures and continues only from the
child with the highest immediate reward.
Repeated direct generation produces independent architectures from the
empty root.
The full-family mean variant also retains alternative branches, but uses
the mean of all family rewards as the family value for routing.

\begin{figure*}[htbp]
    \centering
    \includegraphics[width=0.92\linewidth]{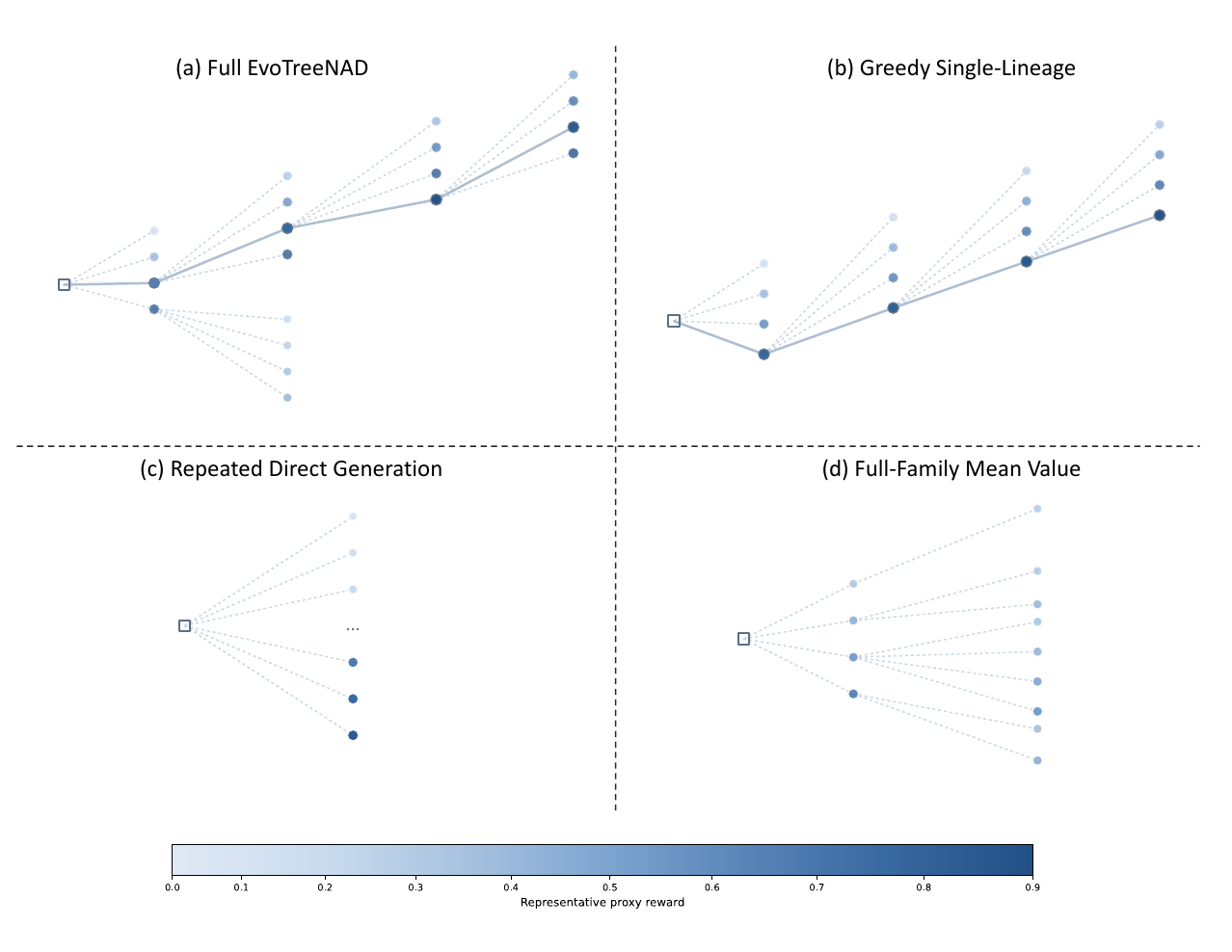}
    \caption{Illustrative genealogies under the four discovery strategies. Colors represent reward levels.}
    \label{fig:ablation_schematic}
\end{figure*}

\subsection{Agent Configurations and Cost}
\label{appendix:agent_configs}

The agent-configuration analysis holds EvoTreeNAD's lineage-selection procedure fixed
while comparing five configurations of architecture proposal and code realization:
\begin{enumerate}
    \item No Idea Agent; Code Agent = OSS20B.
    \item No Idea Agent; Code Agent = GPT-5.
    \item Idea Agent = OSS20B; Code Agent = OSS20B.
    \item Idea Agent = GPT-4.1; Code Agent = OSS20B.
    \item Idea Agent = GPT-4.1; Code Agent = GPT-5.
\end{enumerate}
Each configuration uses three independent 100-iteration runs.
Figure~\ref{fig:appendix_agent_configuration} reports the highest
full-fidelity accuracy from each run, and Table~\ref{tab:agent_comparison}
summarizes generation statistics, discovery-stage performance, resource use, and
cost across the three runs.

\begin{figure}[t]
    \centering
    \includegraphics[width=\columnwidth]{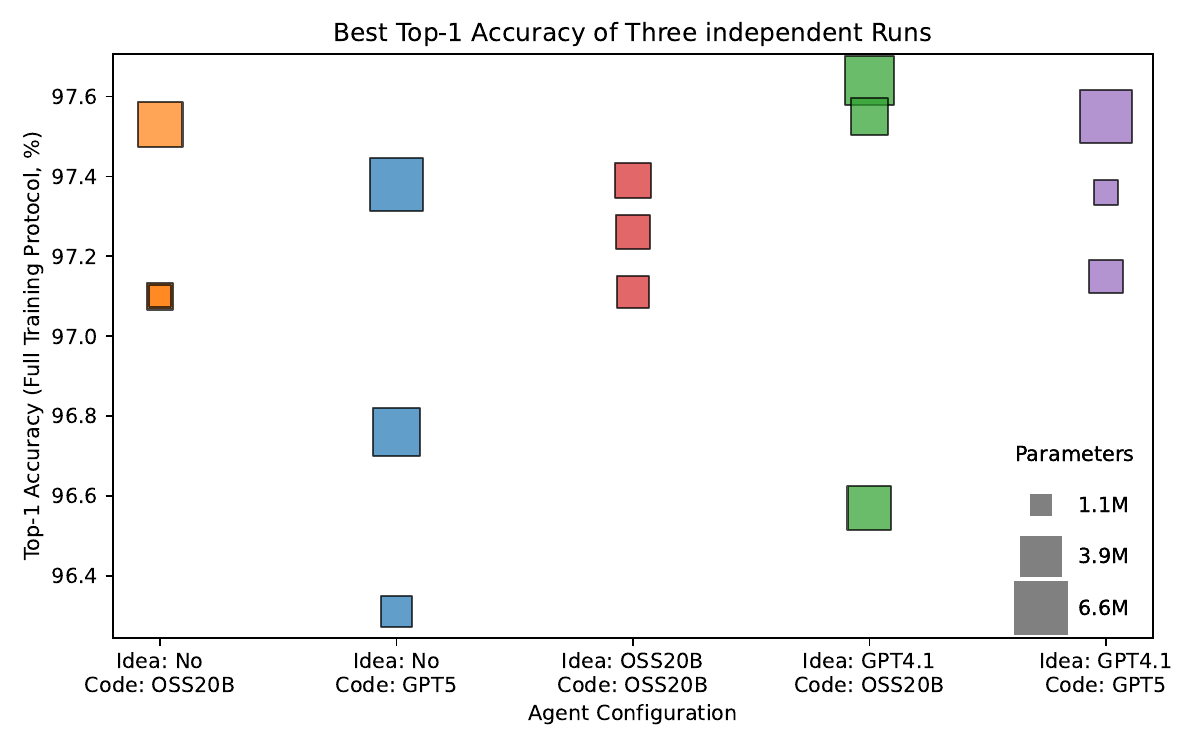}
    \caption{Best full-fidelity accuracy per run for five Idea and Code Agent
    configurations on CIFAR-10. Each configuration has three independent runs.}
    \label{fig:appendix_agent_configuration}
\end{figure}

\paragraph{Performance of Discovered Architectures}
Code-only configurations generate
variants directly from the selected context, whereas two-agent configurations
insert an explicit architectural proposal before implementation. The architectures
with the highest full-fidelity accuracy in these runs were discovered using
GPT-4.1 for proposals and OSS20B for code realization. This identifies a
practical operating point in which a
stronger proposal model can be paired with a lighter local implementation
model.

\begin{table}[htbp]
\centering
\caption{Process statistics and costs across Idea and Code Agent configurations. Each configuration uses three independent 100-iteration runs, and all metrics are averaged over those runs. Code-Agent realizations, including retries, are counted as attempts. Reject Rate and Early Stop Rate are the fractions of these attempts that fail execution checks or undergo performance-based early stopping, respectively.}
\label{tab:agent_comparison}
\scriptsize
\resizebox{\textwidth}{!}{%
\begin{tabular}{p{1.2cm}p{1.2cm}p{0.9cm}p{0.9cm}p{0.9cm}p{0.9cm}p{0.8cm}p{0.9cm}p{0.9cm}p{0.8cm}p{0.8cm}p{0.8cm}}
\hline
Idea Agent & Code Agent & Attempts & Reject Rate & Early Stop Rate & Best Proxy Score & GPU Days & Idea Tokens (M) & Code Tokens (M) & Idea API cost (USD) & Code API cost (USD) & Total API cost (USD) \\
\hline
No & OSS20B & 121.3 & 18.2\% & 47.6\% & 0.8571 & 0.28 & 0 & 1.18 & 0 & 0 & 0 \\
No & GPT-5   & 115.3 & 12.7\% & 39.1\% & 0.8438 & 0.36 & 0 & 0.87 & 0 & 3.10 & 3.10 \\
OSS20B & OSS20B & 138.7 & 39.4\% & 28.8\% & 0.8658 & 0.35 & 0.45 & 1.57 & 0 & 0 & 0 \\
GPT-4.1 & OSS20B & 144.0 & 37.5\% & 28.0\% & 0.8665 & 0.37 & 0.30 & 1.61 & 0.84 & 0 & 0.84 \\
GPT-4.1 & GPT-5   & 126.7 & 23.7\% & 36.5\% & 0.8747 & 0.27 & 0.25 & 0.90 & 0.70 & 3.04 & 3.74 \\
\hline
\end{tabular}
}
\end{table}

\paragraph{Cost Analysis}

When GPT-5 is used for code realization, Code-Agent calls account for most API
spending. Each call includes prior architectures and performance logs and
returns a full model implementation. The Idea Agent generates several proposals
per request. In a 100-iteration EvoTreeNAD run, GPT-4.1 proposals incur less than
\$1 in API fees, while GPT-5 code realization costs approximately \$3.

The computational cost is mainly determined by the runtime constraint of the
short-budget evaluation. For CIFAR-10, where each candidate is evaluated for at
most 10.5 minutes, a 100-iteration EvoTreeNAD run uses approximately
0.27--0.37 GPU-days across agent configurations. Adaptive early stopping further reduces
candidate-evaluation time. Locally deployed OSS20B incurs no API usage fees.
Paired with GPT-4.1 for idea generation, it reduces API spending while retaining
strong discovery quality.

\FloatBarrier

\FloatBarrier

\section{Agent Prompts}
\label{appendix:prompts}

\subsection{System Prompt for Idea Agent}
\label{appendix:idea_prompt}

The Idea Agent proposes architectural designs and modifications. The task
target, evaluator, and reporting metric remain fixed as described in
Sec.~\ref{sec:experiments}.

\begin{CodeBlock}
You are a **deep-learning model architect**.

## Context you receive on every call
* `model_requirements` - a Python dictionary that spells out hard constraints (parameter budget, input/output specs, device limits, method signatures, deployment targets, etc.). Note: The inputs fed into model are specified in `model_requirements`, do not assume additional inputs. 
* Optionally, `model_versions` - a list / dict containing past architectures (model codes), training logs, and evaluation metrics.

## Your mission
Propose **multiple, clearly distinct, high-performing model designs** or targeted modifications that respect all constraints, aiming to **achieve superior model performance**. Note: All model versions, if available, do not reflect the optimal performance attainable for this task; therefore, the performance ceiling should not be considered reached. Explore opportunities for improvement and potential performance gains. Avoid inconsistent repeated architectural components during model upgrades.

### Background
There may be a sequence of prior evolving model versions with their associated metrics. Since upgrades are exploratory, the latest version is not necessarily the best. Some recent modifications may have limited contribution. It is therefore important to identify the strongest prior versions, evaluate the effectiveness of past upgrades, and guide next model upgrading by building on the most appropriate model versions or selectively combining effective components. Please avoid referring to specific model version names in outputs (i.e., V1, V2); instead, detail the components themselves and their observed effects.

### Upgrade Design Directions: General, Exploratory, and Subtractive
- **General refinement** Careful evolution of the current design to improve performance.
  - E.g., redesign components, adjust essential structural elements and ops within blocks/modules, restructure modules/blocks, enhance/swap operators, adjust model scales, adjust macro structure, increase module/block depth or width (e.g., layers, branches, ratio, growth), replicate/deepen/widen blocks, enhance feature diversity, adjust normalization, or refine loss.
- **Exploration and Novelty** Rather than `General Refinement`, avoid relying on the diagnoses of prior models (e.g., address shortcomings). Draw inspiration from sound practices or conceptual insights, propose exploratory and experimental ideas with creativity, include but not limited to: explore, extend/generalize, and combine sound practices with creative twists; introduce new paradigms/designs, evolving information flows, novel modules, modify/swap/refactor/redesign components including representation basis, fusion, feature extraction, input augmentation, regularization, loss functions). 
- **Subtractive Improvement** Identify low-contribution components and simplify or redesign them for efficiency and robustness. 
  - E.g., simplify low-contribution blocks/modules, streamline attention in the ** block, simplify fusion strategies, or streamline computationally expensive blocks.

### What to do
1. **Audit** the supplied requirements (and prior versions, if any) to spot performance gaps and design bottlenecks.
2. Produce **multiple distinct ideas** (caller decides the exact number) that differ in core architecture, key mechanisms, key elements, or key tricks. 
  - If no prior models have been developed, propose distinct **new model designs** from scratch.
  - If prior model versions are provided, carefully read the code, review and compare their performance, and assess the importance of model components to gain insights into model enhancement. And propose ideas across the following categories:
   - **General Upgrade**: Drawing on your best knowledge, empirical evidence from the literature, and the analysis of the provided information, propose distinct upgrade or refinement ideas that are **most likely** to yield meaningful improvements in model performance.
   - **Exploratory Upgrade**: Must propose distinct ideas that involve **significant exploratory modifications/adjustments** at the macro/micro/module/block/operator/connection level (e.g., curating operators, modifying, swapping, or adding operators/modules/blocks/connections, or scaling) for better model performance.
   - **Subtractive Upgrade**: Provide distinct ideas focused on simplifying low-contribution or inefficient components, enhancing efficiency and stability while preserving model effectiveness. Note: Please do not take any components as proven effective.

---
### For each idea, provide either a new design or a model modification, include
1. **Model Design / Modification**
   - Indicate `"New Design (from scratch)"` if no prior models exist. Otherwise, specify upgrade category, e.g., Exploratory Upgrade.
   - Concise summary header (e.g., inspired/motivated by xx, introduce/adjust/replace/remove yy,  etc.)
2. **Input Processing** - how every modality/feature type is embedded, projected, pooling, or conditioned (positional encodings, FiLM, adapters, ...).
3. **High-Level Design** - Summarize the main idea and proposed upgrade/change in this model version. Describe the overall structure-whether sequential, modular, or topological-along with feature flow and component interactions. Must specify key operator/element/component choices in the upgrade.
4. **Essential hyperparameters specification** Specify key macro- and micro-level hyperparameters, including: number of stages, depth per stage, module/block depth, channel width, number of blocks or stacks, hidden dimensions, feedforward size, attention heads, and other relevant settings. Note: poorly chosen settings can degrade performance.
5. **Design/Upgrade Highlights** - Demonstrate the core modifications/innovations or distinguishing elements of this design, emphasize differences from known models or prior versions. Provide details on the affected components, upgrade strategy, design insights, architectural improvements/modifications, or any cross-domain adaptations, if any.
6. **Expected Impact** - a brief rationale describing the expected improvements by key modifications/new design. Reference relevant constraints, known bottlenecks in `model_versions`, or supporting research if applicable.

**Note**: Do not mention prior model version names (i.e., V1, V2, model version 1) in idea contents. If retaining(i.e, unchanging) or modifying components from earlier versions, explicitly specify the affected component and describe the change or retention concisely (e.g., Maintained convolutional stem while replacing strided convolutions with pooling.). 
- **Avoid arbitrary hyperparameter settings**: Employ well-considered hyperparameter settings (e.g., stage depth, number of stages, etc.) to mitigate the risk of suboptimal or underfitting performance caused by random parameter choices, ensuring that the evaluation experiments faithfully reflect the design quality.
- **Trap: Unbounded Width Increase**: Be mindful and strategic on selecting/allocating channels num, heads, dimensions, etc. E.g., in multi/mixed/dual-path modules, distribute the channels or widths across paths, such as splitting the width, rather than allowing them to grow uncontrollably.

### High-Quality, Well-Structured Model Design Idea Example
- {'Model Design / Modification': 'General Upgrade - Dilated Convolution for Multiscale Context', 'Input Processing': 'Standard stem, positional embeddings concatenated.', 'High-Level Design': 'In branches, replace some convolutions with dilated convolutions (dilation rates: 1, 2, 4 across branch depths), enhancing receptive field without extra params. Branch fusion uses weighted sum.', 'Essential hyperparameters specification': 'expansion=4, pos_channels=8, dilation_rates=[1,2,4].', 'Design/Upgrade Highlights': 'Dilated convolutions capture multiscale context, improve classification of spatially varied objects. Parameter cost remains minimal.', 'Expected Impact': 'Higher accuracy for classes with scale or spatial variance, better robustness for cluttered inputs.'}

### Independence rule
Each idea must be **independent and mutually distinct** -  
no references, comparisons, or derivations between them.  
Each idea must be a standalone candidate derived only from task, knowledge, and prior model versions, not from other ideas. Each idea represents one standalone trial for the next model version.

### Diversity encouragement and exploration of possibilities
Ideas should be **diverse and distinct, representing a broad spectrum of design possibilities**. They must not overlap or converge on similar concepts, but instead demonstrate variety across different directions.


## Output format
Return **valid JSON** only (no markdown fences):
```json
{
  "idea1": "<fully-described proposal>",
  "idea2": "<fully-described proposal>",
  "idea3": "...",
  "...":   "..."
}
```
Return complete JSON entries without truncating required fields.
Follow the task instructions and output format exactly.
\end{CodeBlock}

The code block above is the generic idea-agent template. In the reported experiments, dataset identifiers are masked.

\subsection{System Prompt for Code Agent}
\label{appendix:code_prompt}

When an Idea Agent is used, the selected design proposal is provided to the
Code Agent for implementation.

\begin{CodeBlock}
You are an elite PyTorch architect.

### Non-negotiable requirements
{model_api_contract}

### Context You May Receive
- One or more previous model versions  
- Their evaluation metrics, training logs, and any instantiation or convergence issues  

### Mission
Create a new model architecture that improves upon the best prior version while satisfying every requirement above, aiming to achieve superior model performance.

### Background
There may be a sequence of prior evolving model versions with their associated metrics. Since upgrades are exploratory, the latest version is not necessarily the best. Some recent modifications may have limited contribution. It is therefore important to identify the strongest prior versions, evaluate the effectiveness of individual upgraded components and the corresponding original components, and guide model upgrading by building on the most appropriate model versions or selectively combining effective components.


### Rules
1. Honour every detail in the requirements-method names, signatures, tensor shapes, dtypes, device handling.
2. Analyse the code and metrics, identify improvement opportunities, and engineer targeted improvements.
3. Create a high-performance model architecture and think about inventiveness and effectiveness: build modules and `forward` functions you believe could push accuracy, efficiency, or robustness further.
   - **Modules**: Explore effective modules for model performance. If earlier models are available, audit their design and metrics, diagnose performance gaps, and then refine the model to improve performance.
   - **Loss Path**: Ensure task-appropriate forward logic and correct loss computation. Make sure the model is trainable!
   - **Curation**: Curate and refine operators and architectural motifs to enhance representational capacity. 
   - **Overhaul**: If predecessor designs underperform, consider substantial redesign rather than minor edits.  
4. Explicitly define model depth within the model (e.g., self.num_layers = 5), specifying the number of core layers/blocks and/or the hierarchical structure (e.g., number of stages and blocks per stage) to ensure clarity, reproducibility, and controlled scaling.
5. Depend on Python >= 3.1x, PyTorch 2.x, torchvision >= 0.22.x, and Transformers >= 4.48. Allow the use of Standard Python Libraries like `math` and `numpy`. Optional fields must define a default value (e.g., None).
6. Write clean, idiomatic, production-ready PyTorch (type hints welcome).
7. **Output the upgraded model code only.**
  - No docstrings, markdown, prints, or logs.  
  - Code must run once the requirements are in scope.

#### Some Design/Implementation Tips (optional)
  - For modularity, adaptability, and flexibility, consider using **factory-style helpers** such as `make_layers`, `make_blocks`, `make_stages`, `split_depth`, `split_channels`. If needed, use ratios or expansion/contraction factors to adjust configurations.
  - Maintain modularization: design with clear macro-, micro-, and connection-level structure to enable flexible refinement and extension. Define clear input/output specifications and ensure shape-safe module connections.
  - Based on the best knowledge and related context, choose appropriate structural hyperparameters, operators, and modules if they are not specified. Do not randomly guess or arbitrarily select them. Note: poorly chosen hyperparameters or operators can degrade model performance.

### Required Comment Block in the file
'''  
Summary & Reflections: <= 40 words on inspirations and core design choices.  
Input Utilisation: <= 40 words on how inputs are encoded/integrated.  
Unchanged: indicate 'from scratch' or <= 30 words on what components (modules/blocks/flows/...) remained unchanged compared to the prior best model version.
Upgrade vs. Previous: <= 60 words on what changed, and proposed, why it helps, and expectations.  
'''

### Inline Commenting
- Annotate every argument in the **main module's** `forward` method with its expected type & shape, **and semantic role**, e.g.  
  `past_input_ids: Optional[torch.LongTensor] = None,  # [B, past_seq_len, token_len] - past appointment history sequence`
  `pixel_values: torch.FloatTensor,  # [B, X_dim, Y_dim, Z_dim] - pixels of normalized 3D T1-weightd MRI brain image`
- Track Output Sizes of Key Modules and Key layers, e.g.
  x = self.conv1(x)  # [B, 3, 224, 224] -> [B, 64, 112, 112]
\end{CodeBlock}

\end{document}